\documentclass[10pt,twocolumn,twoside]{IEEEtran}  

\usepackage{epsfig}
\usepackage{url}
\usepackage{multirow}
\usepackage[cmex10]{amsmath}
\usepackage{amsmath,amssymb}
\usepackage{graphics}
\usepackage{algorithmic}
\usepackage{array}
\usepackage{mdwmath}
\usepackage{mdwtab}
\usepackage{eqparbox}
\usepackage[tight,footnotesize]{subfigure}
\usepackage{stfloats}
\usepackage{bm}
\usepackage{color}
\usepackage{amsthm}
\usepackage{citesort}

\newtheorem{thm}{Theorem}
\newtheorem{lem}[thm]{Lemma}

\def\x{{\mathbf x}}
\def\L{{\cal L}}

\def\x{\mathbf{x}}
\def\X{\mathbf{X}}
\def\y{\mathbf{y}}

\def\LL{\mathbf{L}}
\def\D{\mathbf{D}}

\def\G{\mathcal{G}}

\def\W{\mathbf{W}}
\def\M{\mathbf{M}}

\def\x{{\mathbf x}}
\def\L{{\cal L}}

\def\u{\mathbf{u}}
\def\v{\mathbf{v}}

\def\x{\mathbf{x}}
\def\X{\mathbf{X}}
\def\y{\mathbf{y}}

\def\L{\mathbf{L}}
\def\D{\mathbf{D}}

\def\TT{\mathbf{T}}

\def\G{\mathcal{G}}

\def\W{\mathbf{W}}
\def\M{\mathbf{M}}

\def\x{{\mathbf x}}

\def\u{\mathbf{u}}

\def\a{\bm{\alpha}}

\def\c{\mathbf{c}}
\def\f{\mathbf{f}}

\def\q{\mathbf{q}}

\def\LL{\mathbf{L}}
\def\D{\mathbf{D}}

\def\G{\mathcal{G}}

\def\U{\mathbf{U}}
\def\Q{\mathbf{Q}}

\def\LL{\mathbf{L}}
\def\FF{\mathbf{F}}

\def\R{\mathbf{R}}
\def\q{\mathbf{q}}
\def\G{{\cal G}}

\def\E{{\cal E}}

\def\VV{{\cal V}}
\def\LL{{\cal L}}

\def\Q{\mathbf{Q}}

\begin{document}

\title{Random Walk Graph Laplacian based Smoothness Prior for Soft Decoding of JPEG Images}

\author{Xianming Liu,~\IEEEmembership{Member,~IEEE},
Gene Cheung,~\IEEEmembership{Senior Member,~IEEE},
Xiaolin Wu,~\IEEEmembership{Fellow,~IEEE},
Debin Zhao,~\IEEEmembership{Member,~IEEE}
\thanks{

X. Liu is with the School of Computer Science and Technology, Harbin Institute of Technology, Harbin, 150001, P.R. China; and also with National Institute of Informatics, 2-1-2, Hitotsubashi, Chiyoda-ku, Tokyo, Japan 101--8430. e-mail: xmliu.hit@gmail.com.

G. Cheung is with National Institute of Informatics, 2-1-2, Hitotsubashi, Chiyoda-ku, Tokyo, Japan 101--8430. e-mail: cheung@nii.ac.jp.

X. Wu is with the Department of Electrical and Computer Engineering, McMaster University, Ontario, Canada, L8S 4K1. e-mail: xwu@ece.mcmaster.ca.

D. Zhao is with the School of Computer Science and Technology, Harbin Institute of Technology, Harbin, 150001, P.R. China. e-mail: dbzhao@hit.edu.cn.
}}

\maketitle

\begin{abstract}
Given the prevalence of JPEG compressed images, optimizing image reconstruction from the compressed format remains an important problem.
Instead of simply reconstructing a pixel block from the centers of indexed DCT coefficient quantization bins (hard decoding), soft decoding reconstructs a block by selecting appropriate coefficient values within the indexed bins with the help of signal priors.
The challenge thus lies in how to define suitable priors and apply them effectively.

In this paper, we combine three image priors---Laplacian prior for DCT coefficients, sparsity prior and graph-signal smoothness prior for image patches---to construct an efficient JPEG soft decoding algorithm.
Specifically, we first use the Laplacian prior to compute a minimum mean square error (MMSE) initial solution for each code block.
Next, we show that while the sparsity prior can reduce block artifacts, limiting the size of the over-complete dictionary (to lower computation) would lead to poor recovery of high DCT frequencies.
To alleviate this problem, we design a new graph-signal smoothness prior (desired signal has mainly low graph frequencies) based on the left eigenvectors of the random walk graph Laplacian matrix (LERaG).
Compared to previous graph-signal smoothness priors, LERaG has desirable image filtering properties with low computation overhead.
We demonstrate how LERaG can facilitate recovery of high DCT frequencies of a piecewise smooth (PWS) signal via an interpretation of low graph frequency components as relaxed solutions to normalized cut in spectral clustering.
Finally, we construct a soft decoding algorithm using the three signal priors with appropriate prior weights.
Experimental results show that our proposal outperforms state-of-the-art soft decoding algorithms in both objective and subjective evaluations noticeably.
\end{abstract}


\IEEEpeerreviewmaketitle

\section{Introduction}
\label{sec:intro}
Millions of images are now captured and viewed daily on social networks and photo-sharing sites like Facebook and Flickr\footnote{It is estimated that 300 million photos are uploaded to Facebook a day.}.
The explosive volume increase of these images are outpacing the cost decrease of storage devices, and thus lossy image compression is still indispensable in today's visual communication systems.
The most prevalent compression format for images remains JPEG (Joint Photographic Experts Group)\footnote{https://jpeg.org/}: a lossy image compression standard whose first and most popular version was finalized more than two decades ago. JPEG is a block-based transform coding scheme, where an image is first divided into non-overlapping $8 \times 8$ pixel blocks, transformed via \textit{discrete cosine transform} (DCT) to coefficients, then lossily quantized and entropy coded.




Broadly speaking, there are two approaches to decoding a JPEG image.
Given encoded quantization bin indices of different DCT coefficients in a pixel block, \textit{hard decoding} chooses the bin centers as reconstructed coefficients and performs inverse DCT to recover the block's pixels.
It is thus inevitable that when the quantization bin sizes increase (coarser quantization), the resulting reconstruction quality will worsen correspondingly.

Instead, one can take a \textit{soft decoding} approach: each DCT coefficient is only constrained to be within the indexed quantization bin, and the reconstruction value is chosen with the help of signal priors.
This is the approach taken in many previous works \cite{avideh,SIAM,DicTV,Liu_2015_cvpr} and is the approach taken in this paper.
The challenge thus lies in identifying appropriate signal priors and incorporating them into an effective soft decoding algorithm.

\begin{figure}[t!]
\centering
\includegraphics[width = 0.45\textwidth]{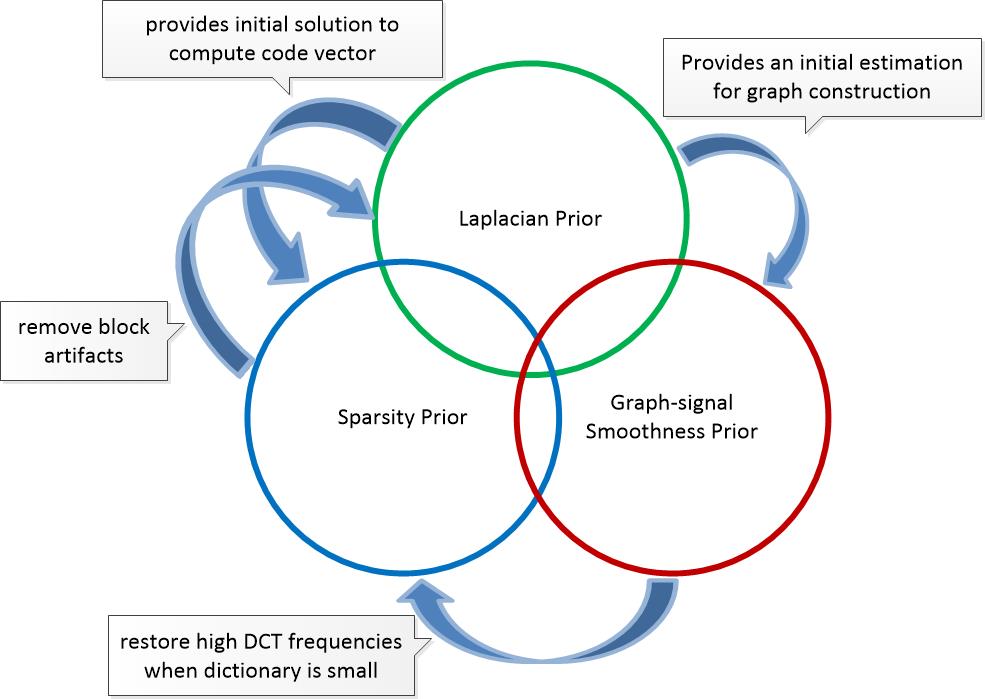}
\caption{Conceptual diagram shows how three image priors complement each other in soft decoding of JPEG images.}
\label{fig:relationship}
\end{figure}

In this paper, we combine three signal priors---Laplacian prior for DCT coefficients \cite{dctlaplacian}, sparsity prior and graph-signal smoothness prior for image patches \cite{sparseoverview,graphoverview}---to build an efficient soft decoding algorithm for JPEG images.
Specifically, for each $8 \times 8$ code block in an image, we first use the Laplacian prior---modeling the probability distributions of DCT coefficients as Laplacian---to compute an initial signal estimate in a closed form that minimizes the expected mean square error (MMSE).
However, optimizing non-overlapping blocks individually can lead to block artifacts.
Using a sparse signal model for a larger patch that straddles multiple $8 \times 8$ code blocks, one can remove block artifacts using an over-complete dictionary of atoms \cite{KSVD} trained offline from a large dataset of natural image patches.

The complexity of computing a code vector is high when the dictionary is large \cite{KSVD}.
Unfortunately, when one reduces the dictionary size, recovery of high DCT frequencies in image patches becomes poor.
We explain this phenomenon by drawing an analogy to low-rate vector quantizers (VQ) \cite{VC} when the DCT frequencies are modeled as Laplacian distributions.

To satisfactorily recover high DCT frequencies, we design a new graph-signal smoothness prior, where the key assumption is that the desired signal (pixel patch) contains mainly low graph frequencies.
Our graph-signal smoothness prior is constructed using \textbf{l}eft \textbf{e}igenvectors of the \textbf{ra}ndom walk \textbf{g}raph Laplacian matrix (LERaG); compared to previous graph-based priors \cite{jiahao,jiahao15,hu16spl}, LERaG has desirable image filtering properties with low computation overhead.
We demonstrate how LERaG can facilitate recovery of high DCT frequencies of piecewise smooth (PWS) signals via an interpretation of low graph frequency components as relaxed solutions to normalized cut in spectral clustering \cite{Ncut}---the closer a target signal is PWS, the more easily the pixels are divided into clusters, and the more likely LERaG can restore the signal.
Finally, we construct a soft decoding algorithm using the three signal priors with appropriate prior weights.
Experimental results show that our proposal outperforms state-of-the-art soft decoding algorithms in both objective and subjective evaluations noticeably.

The outline of the paper is as follows. We first review related works in Section~\ref{sec:related}.
We then discuss the three signal priors in order: the Laplacian prior, the sparsity prior and the graph-signal smoothness prior, in Section~\ref{sec:laplacian}, Section~\ref{sec:formulate} and Section~\ref{sec:graph}, respectively.
In Section~\ref{sec:optimization}, we combine the priors into a JPEG soft decoding algorithm.
Finally, we present results and conclusion in Section~\ref{sec:results} and \ref{sec:conclude}, respectively.

\section{Related Work}
\label{sec:related}

\subsection{Soft Decoding of JPEG Images}

There are two general approaches to reconstruct a compressed JPEG image in the literature: deblocking and dequantization.
The purpose of deblocking is to remove the block artifacts in a JPEG image.
By regarding compression noise as additive white Gaussian noise, deblocking works on decoded JPEG image directly, and recovers it in a similar way as denoising via pre-defined prior models, such as local smoothness \cite{Reeve,PDE}, Gaussian processes \cite{KwonPAMI}, sparsity \cite{Jung}, etc.
However, the non-linearity of quantization operations in JPEG makes quantization noises signal dependent, far from being white and independent.
Inaccurate modeling of compression artifacts limits the restoration performance of deblocking.
Dequantization, also called \emph{soft decoding}, treats JPEG image restoration as an ill-posed inverse problem: find the most probable transform coefficients in a code block subject to indexed quantization bin constraints, with the help of signal priors and optimization.
Soft decoding utilizes all available information such as indexed quantization bin boundaries and natural image priors, and hence has more potential to get better restoration performance. Soft decoding is also the approach taken in this paper. In the following, we review some popular and state-of-the-art methods in the literature.

Many dequantization methods exploit image priors in the pixel or transform domain to address the inverse problem, such as the classical \textit{projection on convex sets} (POCS) method \cite{avideh}, the total variation (TV) model \cite{SIAM}, and non-local self-similarity in DCT domain \cite{ZhangTIP}. Accompanying its success in other restoration applications, the sparsity prior has also shown its promise in combating compression distortions \cite{DicTV,D2SR,SSRQC}.
For instance, Chang \emph{et al.} \cite{DicTV} proposed to learn the dictionary from the input JPEG image, and use total variation and quantization constraint to further limit the solution space.
However, the dictionary trained from JPEG image itself will also learn noise patterns as atoms.
Therefore, in some cases, this method will enhance but not suppress the compressed artifacts.
Very recently, Zhao \emph{et al.} \cite{SSRQC} proposed to combine structural sparse representation prior with quantization constraint.
This scheme achieves state-of-the-art soft decoding performance. As depicted by the theoretical analysis in Section V-C, the sparsity prior cannot provide satisfactory high-frequency information preservation when the dictionary size is small.

\subsection{Graph Laplacian Regularizer}

Leveraging on recent advances in graph signal processing \cite{graphoverview,cmugraph}, graph Laplacain regularizer has shown its superiority  as a new prior in a wide range of inverse problems, such as denoising \cite{hu13,jiahao15}, super-resolution \cite{mao}, bit-depth enhancement \cite{pengfei} and deblurring \cite{PaymanTIP14}. The definition of graph Laplacian regularizer relies on the graph Laplacian matrix, which describes the underlying structure of the graph signal. Two definitions of graph Laplacian matrix are typically used: combinatorial Laplacian and normalized Laplacian.

Most of existing methods \cite{hu13,jiahao15,mao,pengfei} utilize combinatorial graph Laplacian to define the smoothness prior, which is real, symmetric and positive-semidefinite. Its spectrum carries a notion of frequency. However, its filtering strength depends on the degrees of the graph vertices. Several normalized graph Laplacian versions are proposed so that their filtering strength is independent of the vertex degrees. One popular option is to normalize each weight symmetrically \cite{graphoverview}. The symmetrically normalized graph Laplacian has similar properties as the combinatorial Laplacian, except that it does not have the DC component. Therefore, it cannot handle constant signals well. 
\cite{PaymanTIP14} proposed a doubly stochastic graph Laplacian, which is symmetric and contains the DC component. However, it requires non-trivial computation to identify transformation matrices to make the rows and columns of the Laplacian matrix stochastic.

In our previous work \cite{xmicip15}, we used a graph-signal smoothness prior based on combinatorial graph Laplacian for soft decoding of JPEG image. 
In this paper, we propose a new smoothness prior using left eigenvectors of the random walk graph Laplacian matrix to overcome the drawbacks of the combinatorial graph Laplacian. 
Further, we provide a thorough analysis from a spectral clustering perspective to explain why our proposed graph-signal smoothness prior can recover high DCT frequencies of piecewise smooth signals.

\section{Problem Formulation}
\label{sec:background}
\subsection{Quantization Bin Constraint}

We begin with the problem setup. In JPEG standard, each non-overlapping $8 \times 8$ pixel block (called \textit{code block} in the sequel) in an image is compressed independently via transform coding.
Specifically, each $8 \times 8$ code block in vector form $\mathbf{y} \in \mathbb{R}^{64}$ is transformed via DCT to 64 coefficients $\mathbf{Y} = \mathbf{T} \, \mathbf{y}$, where $\mathbf{T}$ is the transform matrix.
The $i$-th coefficient $Y_i$ is quantized using \textit{quantization parameter} (QP) $Q_i$---assigned a quantization bin (\textit{q-bin}) index $q_i \in \mathbb{I}$ (\textit{q-index}) as:
\begin{equation}
q_i = \mathtt{round}\left( {Y_i}/{Q_i}\right).
\end{equation}
Because natural images tend to be smooth, most AC coefficients are close or equal to zero.
JPEG also employs larger QPs for higher frequencies \cite{JPEG} due to human's lesser visual sensitivities to fast changing details, resulting in compact signal representation in the DCT domain for most code blocks.

At the decoder, having received only q-index $q_i$ there exists an uncertainty when recovering $Y_i$, in particular:
\begin{equation}
q_i Q_i \leq Y_i < (q_i+1) Q_i.
\label{eq:qbin}
\end{equation}
This quantization constraint defines the search space for $Y_i$ during restoration of the code block.

\subsection{Soft Decoding of JPEG Images}

A JPEG standard decoder \cite{JPEG} simply chooses the q-bin centers as the reconstructed coefficients and performs inverse DCT to recover the block's pixels.
Clearly, the restoration quality will worsen as the sizes of the q-bins increase, particular at high frequencies.
This approach is called \textit{hard decoding}.
Instead, our approach is to only constrain each DCT coefficient to be within the indexed q-bin using (\ref{eq:qbin}), and choose the ``optimal" reconstruction value using pre-defined signal priors.
This approach is called \textit{soft decoding}.

It is known that block artifacts appear as unnatural high frequencies across block boundaries at low rates \cite{avideh}.
To reduce block artifacts, we enforce consistency among adjacent blocks as follows.
As shown in Fig.~\ref{fig:scheme}, we define a larger \textit{patch} $\x \in \mathbb{R}^n$ that encloses a smaller code block $\y \in \mathbb{R}^{64}$, where $n > 64$.
Mathematically, $\mathbf{y} = \mathbf{M} \x$, where binary matrix $\mathbf{M} \in \{0, 1\}^{64 \times n}$ extracts pixels in $\x$ corresponding to the smaller code block $\mathbf{y}$.
$\x$ is the basic processing unit containing pixels from multiple patches, and thus when recovering $\x$, pixels across patch boundaries are enforced to be consistent by averaging.

\begin{figure}[t!]
\centering
\includegraphics[width = 0.35\textwidth]{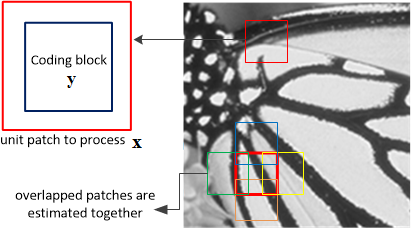}
\caption{A patch being optimized encloses a smaller code block. Pixels across block boundaries are enforced to be consistent by averaging to remove discontinuities.}
\label{fig:scheme}
\end{figure}

The ill-posed nature of soft decoding means that signal prior(s) $p(\mathbf{x})$ is needed to differentiate among good candidate solutions $\mathbf{x}$'s.
We first denote by $\q$ the vector of q-indices corresponding to block $\y$ surrounded by $\x$ (\textit{i.e.}, $\mathbf{y} = \mathbf{M} \x$) and $\Q$ the vector of corresponding QPs.
Having defined $p(\x)$, we pursue a \textit{maximum a posterior} (MAP) formulation and seek the signal $\x$ with the largest posterior probability $p(\x \mid \q)$.
Mathematically, MAP is formulated as:
\begin{equation}
\label{posterior}
\begin{array}{l}
\x^* = \mathop {\arg \max }\limits_{\x}{p\left( {\x \mid \q} \right)} \\
\hspace{0.5cm}= \mathop {\arg \max }\limits_{\x} {p\left( {\q \mid  \x} \right)}
\; p(\x).
 \end{array}
\end{equation}
where $p\left( {\q \mid  \x} \right)$ denotes the likelihood of observing q-indices $\q$ given $\x$.
$p\left( {\q \mid  \x} \right)$ can be written as:
\begin{equation}
\label{eq:likelihood}
p(\q \mid \x) = \left\{ \begin{array}{l}
1 \hspace{0.3cm} \mbox{if} \hspace{0.1cm} \mathtt{round}(\TT\M\x/\Q) = \q \\
0 \hspace{0.3cm} \mbox{o.w.}
\end{array} \right.
\end{equation}
where $\TT\M\x/\Q$ here means element-by-element division.
Thus, the MAP formulation (\ref{posterior}) becomes:
\begin{equation}
\label{obj1}
\begin{array}{l}
\x^* = \mathop {\arg \max }\limits_{\x}{p(\x)}. \\
\hspace{-0.3cm} \mbox{s.t.} \hspace{0.3cm} \q\Q \preceq \mathbf{T}\M\x \prec (\q+1)\Q
 \end{array}
\end{equation}

\section{The Laplacian Prior}
\label{sec:laplacian}
Given that individual DCT coefficients of feasible solutions are constrained to be inside indexed q-bins (\ref{eq:qbin}), one natural choice for signal prior is the Laplacian prior, which is also defined for individual DCT coefficients. Specifically, the Laplacian prior \cite{dctlaplacian} states that the probability density function of a DCT coefficient $Y_i$ is:
\begin{equation}
P_L(Y_i) = \frac{\mu_i}{2}\exp ( - \mu_i \left| Y_i \right|),
\label{eq:uncertain}
\end{equation}
where $\mu_i$ is a parameter. \cite{dctlaplacian} show that the higher the frequency, the sharper the Laplacian distribution (larger $\mu_i$).

Using the Laplacian prior alone, given q-bin constraints we can compute a \textit{minimum mean square error} (MMSE) solution for soft decoding.
In particular, each optimal DCT coefficient $Y_i^*$ of a code block $\y$ in a MMSE sense can be computed independently as:
\begin{equation}
\label{eq:MMSE}
Y_i^* = \arg \min \limits_{{Y_i^o}}
\int_{q_i Q_i}^{(q_i+1)Q_i}
({Y_i^o} - Y_i)^2 \, P_L(Y_i)
~ d Y_i.
\end{equation}
By taking the derivative of (\ref{eq:MMSE}) with respect to $Y_i^*$ and setting it to zero, we obtain a closed-form solution:
\begin{equation}
\label{eq:MMSE-solution}
\small
{Y_i^*} = \frac{{\left( {q_i Q_i} + \mu_i \right)e^{ \left\{ \frac{-q_i Q_i}{\mu_i} \right\}} - \left( {(q_i+1)Q_i} + \mu_i \right) e^{\left\{  \frac{-(q_i+1)Q_i}{\mu_i} \right\}}}}{{e^{ \left\{ \frac{-q_i Q_i}{\mu_i} \right\}} - e^{\left\{  -\frac{(q_i+1)Q_i}{\mu_i} \right\}}}}.
\end{equation}
The recovered code block $\y^*$ can be obtained by performing inverse DCT on the estimated coefficients $\{Y_i^*\}$. After performing MMSE estimation for all code blocks, we get a soft decoded image.

A clear advantage of Laplacian-based soft decoding is that there is a closed-form MMSE solution computed efficiently, which by definition has a smaller expected squared error than a MAP solution, which is just the most probable solution \cite{wan16}.
The closed form is possible because the Laplacian prior describes distributions of individual coefficients, which is also how the q-bin constraints are specified in (\ref{eq:qbin}).
However, the MMSE solution (\ref{eq:MMSE-solution}) can only be used to recover code blocks $\y$ separately, and thus cannot handle block artifacts that occur across adjacent blocks.
As such, we next propose to employ the sparsity prior to provide additional \textit{a priori} signal information and optimize at a larger patch level $\x$.

\section{The Sparsity Prior}
\label{sec:formulate}

Given that it is difficult to use the Laplacian prior directly to recover a larger patch $\x$, we first formulate a MAP problem for a patch using the sparsity prior.
Then, using again the Laplacian prior we analyze the K-SVD based dictionary learning, showing that when the dictionary is small, the atoms tend to have low average DCT frequency.

\subsection{Sparse Signal Model}

The sparsity prior \cite{KSVD} states that a signal $\x \in \mathbb{R}^n$ can be well approximated by a sparse linear combination of atoms $\{ \phi_i \}$ from an appropriately chosen over-complete dictionary $\bm{\Phi} \in \mathbb{R}^{n \times M}$, where the dictionary size is much larger than the signal dimension, \textit{i.e.}, $M \gg n$.
Mathematically we write:
\begin{equation}
\label{eq:sparse}
{\x} = \bm{\Phi}\a + \bm{\xi},
\end{equation}
where \textit{code vector} $\a \in \mathbb{R}^{M}$ contains the coefficients corresponding to atoms $\{ \phi_i \}$ in $\bm{\Phi}$, and $\bm{\xi} \in \mathbb{R}^n$ is a small error term. $\a$ is assumed to be sparse; \textit{i.e.}, the number of non-zero entries in $\a$ is small.
Dictionary $\bm{\Phi}$ can be learned offline from a large set of training patches using K-SVD \cite{KSVD}.

Given signal $\x$, an optimal $\a$ can be found via \cite{KSVD}:
\begin{equation}
\label{eq:L0}
\a^*=\mathop {\arg \min }\limits_{\a} \left\| {{\x} - {\bm{\Phi}}{\a}} \right\|_2^2 + {\lambda\left\| {{\a}} \right\|_0},
\end{equation}
where $\left\| \a \right\|_0$ is the $l_0$-norm of $\a$; $\lambda$ is a parameter trading off the fidelity term with the sparsity prior.
(\ref{eq:L0}) is NP-hard, but the matching and basis pursuit algorithms have been shown effective in finding approximated solutions \cite{KSVD,MatchingPursuits,BasisPursuit}.
For instance, in \cite{KSVD}, the \textit{orthogonal matching pursuit} (OMP) algorithm is employed, which greedily identifies the atom with the highest correlation to the current signal residual at each iteration.
Once the atom is selected, the signal is orthogonally projected to the span of the selected atoms, the signal residual is recomputed, and the process repeats.

We can write the sparsity prior as a probability function:
\begin{equation}
\label{eq:l0}
P_S(\x) \propto \exp ( - \lambda\left\| \a \right\|_0).
\end{equation}

\subsection{Sparsity-based Soft Decoding}

Given the prior probability in (\ref{eq:l0}), we can rewrite the optimization problem in (\ref{obj1}) as follows:
\begin{equation}
\label{eq:obj0}
\begin{array}{l}
\min \limits_{\{\x,\a\}}\left\|\x - \bm{\Phi}\a \right\|_2^2 + \lambda \left\|\a \right\|_0,\\
\mbox{s.t.}~ \q\Q \preceq \mathbf{T}\M\x \prec (\q+1)\Q
\end{array}
\end{equation}
Compared to (\ref{eq:L0}), there is an additional quantization constraint, which is specific to the soft decoding problem.

In (\ref{eq:obj0}), both signal $\x$ and code vector $\a$ are unknown.
We can solve the problem alternately by fixing one variable and solving for the other, then repeat until convergence:

\begin{itemize}

 \item \emph{Step 1--Initial Estimation}:  Initialize $\x^{(0)}$. For example, $\x^{(0)}$ is initialized as the closed-form solution image using the Laplacian prior described in Section \ref{sec:laplacian}.

  \item \emph{Step 2--Sparse Decomposition}: Given $\x^{(t)}$ of iteration $t$, compute the corresponding code vector $\a^{(t)}$ by solving the following minimization problem:
 \begin{equation}
\label{eq:dictionarlearning}
\begin{array}{l}
\a^{(t)} = \arg \min \limits_{\a} \left\| {{\x}^{(t)} - {\bm{\Phi}}{\a}} \right\|_2^2 + \lambda \left\|\a \right\|_0,
\end{array}
\end{equation}
using OMP\footnote{In the unlikely case that there exist multiple $\a$'s that yield  the same minimum objective, we assume that the algorithm returns deterministically the ``first" solution (\textit{e.g.}, one with non-zero coefficients of the smallest atom indices) so the solution to (\ref{eq:dictionarlearning}) is always unique.} stated earlier.


\item \emph{Step 3--Quantization Constraint}: Given sparse code $\a^{(t)}$, optimize $\x^{(t+1)}$ that satisfies the q-bin constraint:
    \begin{equation}
\label{eq:obj_qs}
\begin{array}{l}
\x^{(t+1)} = \arg \min \limits_{\x}\left\|\x - \bm{\Phi}\a^{(t)} \right\|_2^2,\\
\mbox{s.t.}~ \q\Q \preceq \mathbf{T}\M\x \prec (\q+1)\Q
\end{array}
\end{equation}
which can be solved via quadratic programming \cite{Nocedal06}.
\end{itemize}

Step 2 and Step 3 are repeated until $\a$ converges.
In the above procedure, the computational burden mainly lies in the sparse code search step, where the complexity of OMP increases linearly with the number $M$ of dictionary atoms.

We prove local convergence of our sparsity-based soft decoding algorithm via the following lemma.

\begin{lem}\emph{The sparsity-based soft decoding algorithm converges to a local minimum.}
\end{lem}
\begin{proof}
Steps 2 and 3, examining variables $\a^{(t)}$ and $\x^{(t+1)}$ separately, are both minimizing objective in (\ref{eq:obj0}) that is lower-bounded by 0.
At step 2, given $\x^{(t)}$ that was optimized assuming a fixed previous $\a^{(t-1)}$, the algorithm finds $\a^{(t)}$. If $\a^{(t)} \neq \a^{(t-1)}$, then the objective must be smaller; otherwise, $\a^{(t)}$ is not the objective-minimizing argument, a contradiction.
The same is true for step 3.
Hence the objective is monotonically decreasing when the variables are updated to different values.
Given that the objective is lower-bounded by $0$, the objective cannot decrease indefinitely, and hence the algorithm converges to a local minimum.
\end{proof}

\subsection{Limitation of Small K-SVD Trained Dictionary}

\begin{figure}[t!]
\centering
\includegraphics[width = 0.45\textwidth]{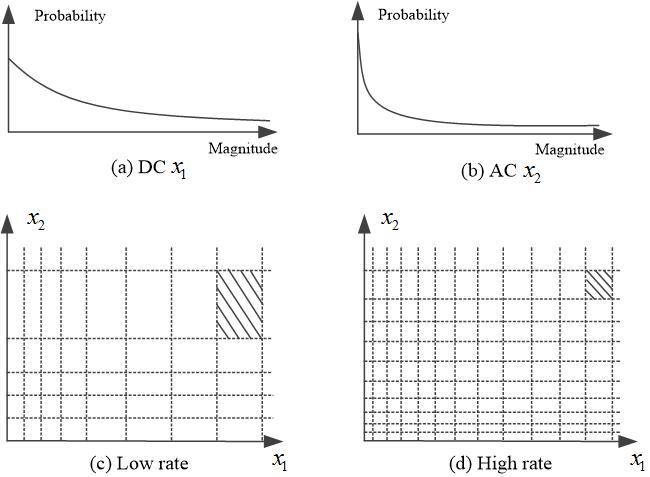}
\caption{Illustration of product VQ for two frequencies: $x_1$ (DC) and $x_2$ (AC).
(a) and (b) show probability distributions of $x_1$ and $x_2$ respectively.
When the dictionary is small (low rate), (c) shows that quantization is coarser for large magnitude in $x_2$ than $x_1$, since the probability of large magnitude in high frequency is relatively low.
When the dictionary is large (high rate), (d) shows that there are enough quantization bins so that quantization for large magnitude in high frequency is sufficiently fine. }
\label{fig:LM}
\vspace{-0.2cm}
\end{figure}

We now argue that when the size of a dictionary trained by K-SVD \cite{KSVD} is small (to reduce complexity during code vector search via OMP), the recovery of high DCT frequencies suffers.
Given a set of $N$ training patches $\{\x_i\}_{i=1}^N$ each of size $\sqrt n  \times \sqrt n$, dictionary $\bm{\Phi} = \{\phi_i\} \in \mathbb{R}^{n \times M}$ of $M$ atoms, is trained using K-SVD via the following formulation:
\begin{equation}
\mathop {\min }\limits_{{\bf{\Phi}} ,\{ {\a_i}\} } \sum\limits_{i = 1}^N {\left\| \x_i - {\bf{\Phi}}\a_i \right\|_2^2}  + \lambda {\left\| \a_i  \right\|_0},
\label{eq:dictTrain}
\end{equation}
The variables are \textit{both} $\bm{\Phi}$ \textit{and} $\{ \a_i \}$, which are solved alternately by fixing one and solving for the other.

We write (\ref{eq:dictTrain}) in the DCT frequency domain using the Parsavel's theorem \cite{RayleighEnergy}:
\begin{equation}
\label{eq:eq_freq}
\mathop { \min }\limits_{{\bf{\Phi}} ,\{ {\a_i}\} } \sum\limits_{i = 1}^N {\left\| \X_i - \mathbf{T}' {\bf{\Phi}}\a_i \right\|_2^2} + \lambda {\left\| \a_i  \right\|_0},
\end{equation}
where $\mathbf{T}'$ is the DCT transform for the $n$-pixel patch, and $\X_i = \mathbf{T}' \x_i$ is the vector of DCT coefficients for $\x_i$.

We rewrite the optimization to its constrained form:
\begin{equation}
\mathop { \min }\limits_{{\bf{\Phi}} ,\{ {\a_i}\} } \sum\limits_{i = 1}^N {\left\| \X_i - \mathbf{T}' {\bf{\Phi}}\a_i \right\|_2^2},\hspace{0.2cm}\mbox{s.t.},\hspace{0.2cm} \left\| \a_i  \right\|_0 \le K
\end{equation}
where $K$ is a pre-set sparsity limit for each $\a_i$.

For intuition, we focus on the special case when $K=1$.
In this special case, the dictionary learning problem is analogous to the \textit{vector quantization} (VQ) design problem \cite{VC}.
Specifically, selecting $M$ atoms in dictionary $\bf{\Phi}$ is analogous to designing $M$ partition $\R_1, \ldots, \R_M$ so that their union is the space of feasible signal $\R$, \textit{i.e.,} $\R = \cup_{m=1}^M \R_m$, and no two partitions overlap, \textit{i.e.} $\R_i \cap \R_j = \emptyset$, $\forall i \neq j$.
The \textit{centroid} of each cell $\R_m$ is atom $\phi_m$.

When the number of training samples $N$ tends to infinity, the sum over $N$ training patches $\X_i$ is replaced by integration over $\X$ each with probability $P(\X)$:
\begin{equation}
\label{eq:eq_freq1}
\mathop {\min }\limits_{\{\phi_m\}} \sum\limits_{m = 1}^M \int\nolimits_{{\R_m}}{\left\| \X - \mathbf{T}'\phi_m \right\|_2^2}P(\X)d\X,
\end{equation}
where a signal $\X$ in partition $\R_m$ (DCT frequency domain) is represented by atom $\phi_m$ (in pixel domain), thus distortion $\| \X - \mathbf{T}'\phi_m\|_2^2$.

We have assumed in Section \ref{sec:laplacian} that $P(\X)$  is a product of Laplacian distributions for individual DCT frequencies, where for high frequencies the distributions are more skewed and concentrated around zero.
Thus, as the number $M$ of atoms $\phi_m$ in dictionary $\bf{\Phi}$ (hence the number of partitions $\R_m$) decreases, in order to minimize the expected squared error in (\ref{eq:eq_freq1}), quantization bins will be relatively coarse for large magnitude of high frequencies, because they are less probable in comparison \cite{swaszek91}.
See Fig.~\ref{fig:LM} for an example that shows coarse quantization bins for $x_2$ for a product VQ of two dimensions.
This explains why when the dictionary $\bf{\Phi}$ is small, recovery of large magnitude high DCT frequencies is difficult.

\begin{figure}[t!]
\centering
\includegraphics[width = 0.5\textwidth]{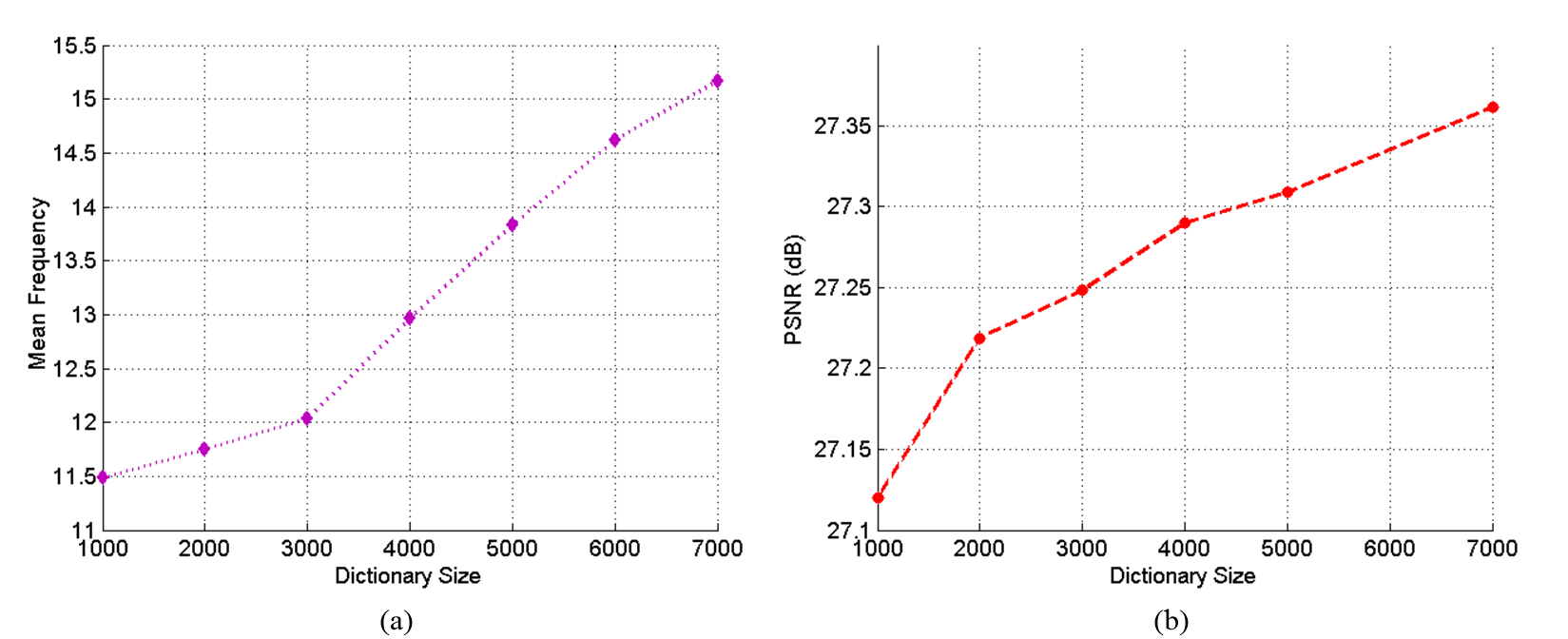}
\caption{(a) The relationship between the dictionary size and the high-frequency information preservation, where the mean frequency is used to measure the high-frequency information contained in atoms of dictionary. (b) The relationship between the dictionary size and the restoration performance, where we use the test image \emph{Butterfly} when QF = 10 as an example. }
\label{fig:trend}
\end{figure}

\begin{figure}[t!]
\centering
\includegraphics[width = 0.45\textwidth]{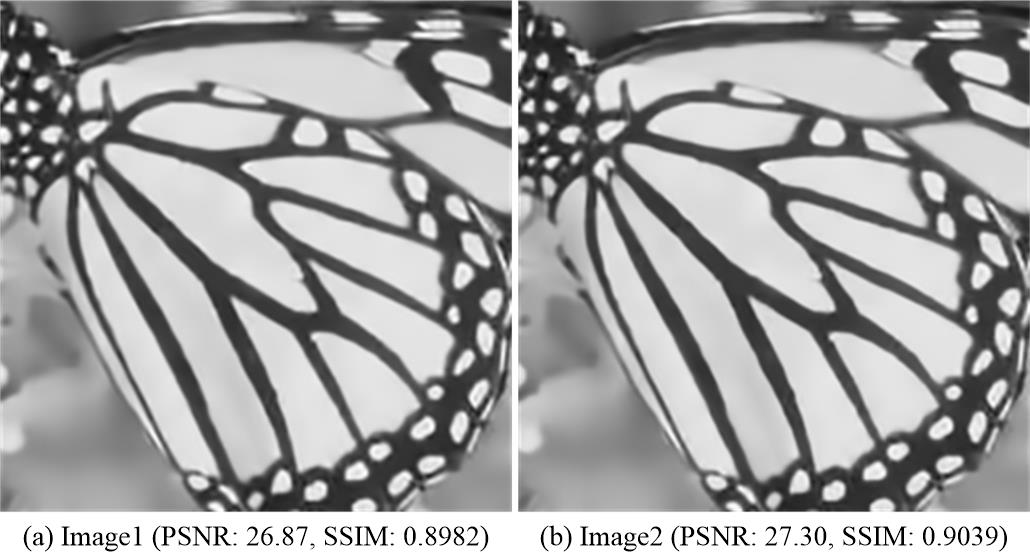}
\caption{Illustration of reconstructed images using dictionaries with different sizes. (a) Image1: the reconstructed image using dictionaries with size 1000. (b) Image2: the reconstructed image using dictionaries with size 7000. (c) The difference between two reconstructed images. The quality improvement of Image2 over Image1 mainly lies on edge structures. }
\label{fig:diff}
\vspace{-0.2cm}
\end{figure}

We can observe empirically this is indeed the case.
For a given K-SVD trained dictionary $\bf{\Phi}$, we compute the \textit{mean frequency} (MF) of atoms as follows.
Let $f_i$ and $Y_i(m)$ be a DCT frequency and an atom $\phi_m$'s corresponding coefficient at that frequency.
MF for $\bf{\Phi}$ is computed as:
\begin{equation}
MF = \frac{1}{M}\sum\limits_{m = 1}^M {\sum\limits_{i = 1}^n {{f_i}Y_i^2(m)} }.
\end{equation}

We computed MF for dictionaries of sizes ranging from $1000$ to $7000$ (the dimension of an atom is $100$ if $\x$ is set to $10 \times 10$).
As illustrated in Fig.~\ref{fig:trend} (a), MF is roughly linearly increasing with dictionary size.
Further, as illustrated in Fig.~\ref{fig:trend} (b), the performance of sparsity-based soft decoding also improves with dictionary size.
As shown in Fig.~\ref{fig:diff}, we can observe that using a large dictionary can significantly improve PSNR and SSIM values of the restored image, where the improvement over a smaller dictionary mainly lies in edge structures.

\section{The Graph-Signal Smoothness Prior}
\label{sec:graph}
Given that the sparsity prior cannot recover high DCT frequencies when QP is large and the dictionary is  small, we introduce the graph-signal smoothness prior, which was shown to restore piecewise smooth (PWS) signals well \cite{jiahao,jiahao15,hu16spl}.
We first define a discrete signal on a graph, then define a conventional smoothness notion for graph-signals and interpret graph-signal smoothness in the graph frequency domain.
Finally, we propose a new graph-signal smoothness prior based on the random walk graph Laplacian for soft decoding.

\subsection{Signal on Graphs}

We define a graph as $\mathcal{G} = \left(\VV, \E, \W\right)$.
$\VV$ and $\E$ are respectively the sets of vertices and edges in $\mathcal{G}$.
$W_{i,j}$ in the adjacency matrix $\mathbf{W}$ is the weight of the edge connecting vertices $i$ and $j$.
We consider only undirected graphs and non-negative edge weights \cite{graphoverview}; \textit{i.e.}, $W_{i,j} = W_{j.i}$ and $W_{i,j} \geq 0$.

A graph-signal $\x \in \mathbb{R}^n$ is a collection of discrete samples on vertices $\VV$ of a graph $\mathcal{G}$, where $n = |\VV|$.
In our work, for each target pixel patch we construct a fully-connected graph, where each vertex (pixel) is connected to all other vertices in $\mathcal{G}$.
The edge weight $W_{i,j}$ between two vertices $i$ and $j$ is conventionally computed using Gaussian kernels \cite{graphoverview}:
\begin{equation}
W_{i,j} = \exp \left( - \frac{\left\| x_i - x_j \right\|_2^2 }{\sigma_1^2} \right) \;
\exp \left( - \frac{\left\| l_i - l_j \right\|_2^2 }{\sigma_2^2} \right)
\label{eq:weight}
\end{equation}
where $x_i$ and $l_i$ are the intensity and location of pixel $i$, and $\sigma_1$ and $\sigma_2$ are chosen parameters.
This definition is similar to domain and range filtering weights in bilateral filter \cite{BilateralFilter}.


\subsection{Graph Laplacians and Signal Smoothness}

A \textit{degree matrix} $\mathbf{D}$ is a diagonal matrix where $D_{i,i} = \sum_j W_{i,j}$.
A \textit{combinatorial graph Laplacian} $\mathbf{L}$ is \cite{graphoverview}:
\begin{equation}
\mathbf{L} = \mathbf{D} - \mathbf{W}.
\end{equation}
There exist two normalized variants of $\mathbf{L}$: \textit{normalized graph Laplacian} $\mathcal{L}_n$ and \textit{random walk graph Laplacian} $\mathcal{L}_r$,
\begin{equation}
\label{eq:normL}
\mathcal{L}_n = \mathbf{D}^{-1/2}\mathbf{L}\mathbf{D}^{-1/2}, ~~~
\mathcal{L}_r = \mathbf{D}^{-1} \mathbf{L}.
\end{equation}
A normalized version means that its filtering strength does not depend on the degrees of the graph vertices.

Given $\mathbf{L}$, one can describe the squared variations of the signal $\x$ with respect to $\mathcal{G}$ using the \textit{graph Laplacian regularizer} $\x^T \mathbf{L} \x$ \cite{jiahao,jiahao15}:
\begin{equation}
\label{eq:GL}
\x^T\mathbf{L}\x = \frac{1}{2} \sum_{(i,j) \in \mathcal{E}}
{\left( x_i - x_j \right)}^2 {W_{i,j}} .
\end{equation}

A common graph-signal smoothness prior can then be defined as:
\begin{equation}
\label{eq:graphprior}
\begin{array}{l}
P_G(\x) \propto \exp \left( { - \lambda_2{\x^T} \mathbf{L} \x} \right),
\end{array}
\end{equation}
which states that a probable signal $\x$ has small $\x^T\mathbf{L}\x$.
By (\ref{eq:GL}), $\x^T\mathbf{L}\x$ is small if $\x$ has similar values at each vertex pair $(i,j)$ connected by an edge, \textit{or} the edge weight $W_{i,j}$ is small.
Thus $\x$ is smooth if its variations across connected vertices are consistent with the underlying graph (where edge weights reflect expected signal variations).
For example, if a discontinuity is expected at pixel pair $(x_i, x_j)$, one can pre-set a small weight $W_{i,j}$, so that the smoothness prior will not over-smooth during optimization.
Later, we will propose a different smoothness prior than the conventional (\ref{eq:graphprior}) with more desirable filtering properties.


\subsection{Frequency Interpretation of Graph Smoothness Prior}

Since $\mathbf{L}$ is a real, symmetric, positive semi-definite matrix, it can be decomposed into a set of orthogonal eigenvectors, denoted by $\{\u_i\}_{i=1,\cdots, n}$, with real non-negative eigenvalues $0 = \eta_1 \leq \eta_2 \le \cdots \le \eta_{n}$.
We define $\U$ as the eigen-matrix with $\u_i$'s as columns and $\bm{\Lambda}$ as the diagonal matrix with $\eta_i$'s on its diagonal.
$\mathbf{L}$ can thus be written as:
\begin{equation}
\label{eq:Ldec}
\begin{array}{l}
\mathbf{L} = \U\bm{\Lambda}\U^T.
\end{array}
\end{equation}
We define the \textit{graph Fourier transform} (GFT) matrix as $\FF = \U^T$. A graph-signal $\x$ can be transformed to the graph frequency domain via:
\begin{equation}
\label{eq:GFT}
\begin{array}{l}
\boldsymbol{\alpha} = \FF\x.
\end{array}
\end{equation}
With this definition, the graph Laplacian regularizer can be written as:
\begin{equation}
\label{eq:graphfre}
\begin{array}{l}
\x^T\mathbf{L}\x =  \boldsymbol{\alpha}^T\bm{\Lambda}\boldsymbol{\alpha}
= \sum\limits_{k}{{\eta_k} \, \alpha_k^2}.
\end{array}
\end{equation}
$\x^T\mathbf{L}\x$ can thus be written as a sum of squared GFT coefficients $\alpha_k^2$ each scaled by the eigenvalue $\eta_k$, which is interpreted as frequency in graph transform domain.

By maximizing the graph-signal smoothness prior in (\ref{eq:graphprior}) (minimizing $\x^T\mathbf{L}\x$), signal $\x$ is smoothened with respect to the graph, \textit{i.e.}, high graph frequencies are suppressed while low graph frequencies are preserved.
It is shown \cite{jiahao,jiahao15,hu16spl} that PWS signals---with discontinuities inside signals that translate to high DCT frequencies---can be well approximated as linear combinations of low graph frequencies for appropriately constructed graphs.
We provide an explanation next, which leads naturally to the derivation of a new smoothness prior.

\subsection{Building A Random Walk Graph Laplacian Prior}

We first show that low graph frequencies for the normalized graph Laplacian $\mathcal{L}_n$ can be interpreted as relaxed solutions for spectral clustering, which are PWS in general.
We then argue that to induce desirable filtering properties for $\mathcal{L}_n$, a similarity transformation is required, resulting in the random walk Laplacian $\mathcal{L}_r$.
Finally, we show that a more appropriate smoothness prior than (\ref{eq:graphprior}) can be defined using $\mathcal{L}_r$, with many desirable filtering properties.

\subsubsection{Spectral Clustering}

To understand low graph frequencies for $\mathcal{L}_n$ computed from a PWS signal $\mathbf{x}$, we take a \textit{spectral clustering} perspective.
Spectral clustering \cite{spectralclustering} is the problem of separating vertices $\mathcal{N}$ in a similarity graph $\mathcal{G}=(\mathcal{N}, \mathcal{E}, \mathbf{W})$ into two subsets of roughly the same size via spectral graph analysis.
Specifically, the authors in a landmark paper \cite{Ncut} proposed a minimization objective \textit{normalized cut} (Ncut) to divide $\mathcal{N}$ into subsets $\mathcal{A}$ and $\mathcal{B}$:
\begin{equation}
\min_{\mathcal{A}, \mathcal{B}} ~
\mathrm{Ncut}(\mathcal{A}, \mathcal{B}) := \mathrm{cut}(\mathcal{A}, \mathcal{B})
\left( \frac{1}{\mathrm{vol}(\mathcal{A})} + \frac{1}{\mathrm{vol}(\mathcal{B})} \right),
\label{eq:ncut}
\end{equation}
where $\mathrm{cut}(\mathcal{A}, \mathcal{B})$ counts the edges across the two subsets $\mathcal{A}$ and $\mathcal{B}$, and $\mathrm{vol}(\mathcal{A})$ counts the degrees of vertices in $\mathcal{A}$:
\begin{equation}
\mathrm{cut}(\mathcal{A},\mathcal{B}) =
\sum\limits_{i \in \mathcal{A}, j \in \mathcal{B}} {{W_{i,j}}},~~
\mathrm{vol}(\mathcal{A}) =  \sum\limits_{i \in \mathcal{A}} D_{i,i}.
\end{equation}
From (\ref{eq:ncut}), it is clear that a good division of $\mathcal{N}$ into similar clusters $\mathcal{A}$ and $\mathcal{B}$ with small $\mathrm{Ncut}(\mathcal{A}, \mathcal{B})$ will have small $\mathrm{cut}(\mathcal{A}, \mathcal{B})$ and large $\mathrm{vol}(\mathcal{A})$ and $\mathrm{vol}(\mathcal{B})$.
However, minimizing Ncut over all possible $\mathcal{A}$ and $\mathcal{B}$ that divide $\mathcal{N}$ is NP-hard.

Towards an approximation, the authors first rewrote the minimization of Ncut as:
\begin{equation}
\min_{\mathcal{A}, \mathcal{B}} ~ \mathrm{Ncut}(\mathcal{A}, \mathcal{B}) =
\min_{\mathbf{f}} \frac{{{\f^T}\L\f}}{{{\f^T}\D\f}},
\label{eq:relaxed_ncut}
\end{equation}
where
\begin{equation}
\f= [f_1,\cdots,f_n]^T \hspace{0.2cm} \mbox{and} \hspace{0.2cm}
{f_i} = \left\{ \begin{array}{ll}
\frac{1}{\mathrm{vol}(\mathcal{A})}  & \mbox{if} \hspace{0.2cm} i \in \mathcal{A} \\
\frac{{ - 1}}{\mathrm{vol}(\mathcal{B})} & \mbox{if} \hspace{0.2cm} i\in \mathcal{B}
\end{array} \right.
\label{eq:relaxed_ncut2}
\end{equation}
We see that $\mathbf{f}$ is \textit{piecewise constant} (PWC).

The problem constraints (\ref{eq:relaxed_ncut2}) are then relaxed, resulting in:
\begin{equation}
\min_{\mathbf{f}} \frac{{{\f^T}\L\f}}{{{\f^T}\D\f}},
\hspace{0.2cm}\mbox{s.t.} \hspace{0.2cm} \f^T\D\mathbf{1} = 0.
\label{eq:relaxed_ncut3}
\end{equation}
The idea is that if solution $\mathbf{f}^*$ satisfies (\ref{eq:relaxed_ncut2}), then $\mathbf{f}^{* T} \mathbf{D} \mathbf{1} = 0$.
Thus a solution $\mathbf{f}^*$ to (\ref{eq:relaxed_ncut3}) can be interpreted as a relaxed / approximate solution to the Ncut problem in (\ref{eq:ncut}).

We now use this result from \cite{Ncut} to interpret low graph frequencies of $\mathcal{L}_n$.
We first define $\v := \D^{1/2}\f$ and $\mathbf{v}_1 := \D^{1/2} \mathbf{1}$.
(\ref{eq:relaxed_ncut3}) can then be rewritten as:
\begin{equation}
\min_{\mathbf{v}}
\frac{{{\v^T}\mathcal{L}_n\v}}{{{\v^T}\v}},\hspace{0.2cm}\mbox{s.t.}
\hspace{0.2cm} \v^T \v_1 = 0.
\label{eq:relaxed_cut4}
\end{equation}
Note that the objective in (\ref{eq:relaxed_cut4}) is the \textit{Rayleigh quotient} with respect to matrix $\mathcal{L}_n$ \cite{MatrixAnalysis}.
It is easy to see that $\v_1$ minimizes this objective:
$\v_1^T \mathcal{L}_n \v_1 = \mathbf{1}^T \D^{1/2} \D^{-1/2} \mathbf{L} \D^{-1/2} \D^{1/2} \mathbf{1} = \mathbf{1}^T \mathbf{L} \mathbf{1} = 0$.
Hence $\v_1$ is the first eigenvector of $\mathcal{L}_n$, and the solution to (\ref{eq:relaxed_cut4}) is the second eigenvector $\v_2$ of $\mathcal{L}_n$.
We can thus conclude the following:
\begin{quote}
\textit{The second eigenvector $\v_2$ of $\mathcal{L}_n$ is a relaxed solution to the Ncut problem; if the solution becomes exact, then $\v_2$ is PWC according to (\ref{eq:relaxed_ncut2})}.
\end{quote}
Eigen-basis of $\mathcal{L}_n$ thus seem suitable to compactly represent PWS signals.

\subsubsection{Smoothness Prior using Random Walk Graph Laplacian}

However, because the first eigenvector of $\mathcal{L}_n$, $\mathbf{v}_1 = \D^{1/2} \mathbf{1}$, is not a constant vector (DC component) in general, the eigen-basis of $\mathcal{L}_n$ are not suitable for filtering of images, which tend to be smooth.
We thus perform a \textit{similarity transformation}\footnote{https://en.wikipedia.org/wiki/Matrix\_similarity} on $\mathcal{L}_n$ to efficiently filter constant signals.
Let $\mathcal{L}_n = \mathbf{V} \mathbf{\Lambda} \mathbf{V}^T$, where $\mathbf{V}$ is a matrix with eigenvectors of $\mathcal{L}_n$ as columns, and $\mathbf{\Lambda}$ is a diagonal matrix with corresponding eigenvalues $\tilde{\eta}_k$ on its diagonal.
We now define $\LL_r := \D^{-1/2}\LL_n \D^{1/2} = \D^{-1}\L$.
We see that $\LL_r$ has left eigenvectors $\mathbf{V}^T \D^{1/2}$ (in rows):
\begin{equation}
\mathbf{V}^T \D^{1/2} \mathcal{L}_r = \mathbf{V}^T \D^{1/2} \D^{-1/2} \mathbf{V} \mathbf{\Lambda} \mathbf{V}^{T} \D^{1/2} =  \mathbf{\Lambda} \mathbf{V}^T \D^{1/2}.
\end{equation}

Note that because $\LL_r$ and $\LL_n$ are similar, they have the same eigenvalues $\tilde{\eta}_k$.
When a constant signal $\mathbf{1}$ is projected onto left eigenvectors $\mathbf{V}^T \D^{1/2}$, $\D^{1/2} \mathbf{1}$ is in fact $\v_1$ in $\mathbf{V}^T$, which is orthogonal to other rows of $\mathbf{V}^T$.
This means that \textit{left eigen-basis coefficients of $\mathcal{L}_r$}, defined as $\boldsymbol{\beta} = \mathbf{V}^T \D^{1/2} \x$, has non-zero only in the first element when $\x = \mathbf{1}$.

However, $\LL_r$ is asymmetric, which cannot be easily decomposed into a set of orthogonal eigenvectors with real eigenvalues.
Thus there is no clear graph frequency interpretation of $\x^T \LL_r \x$ as we did for $\x^T \mathbf{L} \x$ in (\ref{eq:graphfre}).

To achieve symmetry, suppose we use $\LL_r^T\LL_r$ and define the following regularization term instead:
\begin{equation}
\begin{array}{ll}
\x^T\LL_r^T\LL_r\x = \x^T \D^{1/2}\LL_n\D^{-1/2}\D^{-1/2}\LL_n\D^{1/2}\x \\
\hspace{1.6cm}= (\x^T\D^{1/2}\LL_n)\D^{-1}(\LL_n\D^{1/2}\x).
\end{array}
\label{eq:new_regularizer}
\end{equation}
If we write $\boldsymbol{\gamma} = \LL_n\D^{1/2}\x$, then (\ref{eq:new_regularizer}) becomes:
\begin{equation}
\begin{array}{ll}
\x^T\LL_r^T\LL_r\x = \boldsymbol{\gamma}^T\D^{-1}\boldsymbol{\gamma}.
\end{array}
\label{eq:new_regularizer2}
\end{equation}
Since $\D^{-1}$ is a diagonal matrix, one can see that:
\begin{equation}
\frac{{{\boldsymbol{\gamma} ^T}\boldsymbol{\gamma} }}{{{d_{\max }}}}
\le \boldsymbol{\gamma}^T\D^{-1}\boldsymbol{\gamma}
\le \frac{{{\boldsymbol{\gamma} ^T}\boldsymbol{\gamma} }}{{{d_{\min }}}},
\label{eq:new_regularizer2}
\end{equation}
where $d_{\max}$ and $d_{\min}$ are maximum and minimum vertex degrees in graph $\mathcal{G}$.
So instead of minimizing $\boldsymbol{\gamma}^T\D^{-1}\boldsymbol{\gamma}$, we can minimize its upper bound $(d_{\min}^{-1}){\boldsymbol{\gamma} ^T}\boldsymbol{\gamma}$.

We can interpret $\boldsymbol{\gamma}^T \boldsymbol{\gamma}$ more naturally as follows.
$\boldsymbol{\gamma}^T \boldsymbol{\gamma}$ can now be written as:
\begin{align}
\boldsymbol{\gamma}^T \boldsymbol{\gamma} & = \x^T \D^{1/2} \mathbf{V} \mathbf{\Lambda} \mathbf{V}^T
\mathbf{V} \mathbf{\Lambda} \mathbf{V}^T \D^{1/2} \x \nonumber \\
& = \boldsymbol{\beta}^T \mathbf{\Lambda}^2 \boldsymbol{\beta}
= \sum_k \tilde{\eta}_k^2 \, \beta_k^2.
\label{eq:freqInterReg}
\end{align}
Thus we can now have a graph frequency interpretation of our regularizer $(d_{\min}^{-1}) \boldsymbol{\gamma}^T \boldsymbol{\gamma}$: high frequencies of $\mathcal{L}_n$ (or $\mathcal{L}_r$) are suppressed to restore smooth signal $\x$.
For convenience, we call our regularizer $(d_{\min}^{-1}) \boldsymbol{\gamma}^T \boldsymbol{\gamma}$ \textit{\textbf{L}eft \textbf{E}igenvector \textbf{Ra}ndom-walk \textbf{G}raph Laplacian} (LERaG).

For efficient computation, $\boldsymbol{\gamma}^T \boldsymbol{\gamma}$ can be most efficiently computed as:
\begin{equation}
\boldsymbol{\gamma}^T \boldsymbol{\gamma} = \x^T \mathbf{L} \D^{-1} \mathbf{L} \x.
\label{eq:simpReg}
\end{equation}

Compared to $\x^T \mathbf{L} \x$ in (\ref{eq:graphprior}), LERaG is based on the normalized variant $\mathcal{L}_r$, and thus its filtering is insensitive to the vertex degrees of particular graphs.
Further, unlike $\x^T \mathcal{L}_n \x$, LERaG is based on left eigenvectors of $\mathcal{L}_r$ that efficiently filter constant signals, thus is suitable for image filtering.
Finally, compare to \cite{PaymanTIP14} which requires non-trivial computation to identify transformation matrices to make $\mathcal{L}_n$ row and column stochastic, LERaG can be computed simply via (\ref{eq:simpReg}).
Thus our regularizer has desirable filtering properties with little computation overhead compared to previous graph-based regularizers in the literature.

\subsubsection{Analysis of Ideal Piecewise Constant Signals}

We now show that LERaG computes to 0 for ideal PWC signal $\x = [x_1,\cdots,x_n]^T$, where:
\begin{equation}
x_i = \left\{ \begin{array}{ll}
c_1 & \mbox{if}~ 1 \leq i \leq l \\
c_2 & \mbox{if}~ l < i \leq n
\end{array} \right.
\label{eq:pwc_sig}
\end{equation}
where constants $c_1$ and $c_2$ are sufficiently different, \textit{i.e.}, $|c_1 - c_2| > \Delta$ for some sufficiently large $\Delta$.

Using (\ref{eq:weight}) to compute edge weights, this PWC signal implies that $W_{i,j} = 0$ if $i$ and $j$ belong to different constant pieces.
Assume further that $\sigma_2$ is sufficiently large relative to location differences $\| l_i - l_j\|^2$, so that the first exponential kernel in (\ref{eq:weight}) alone determines $W_{i,j}$.
The adjacency matrix $\mathbf{W}$ and the graph Laplacian $\L$ are then both block-diagonal:

\vspace{-0.1in}
\begin{small}
\begin{equation}
\mathbf{W} = \left[ \begin{array}{cc}
\mathbf{A}_{l} & \mathbf{0}_{l \times (n-l)} \\
\mathbf{0}_{(n-l) \times l} & \mathbf{A}_{n-l}
\end{array} \right],
~
\L = \left[ \begin{array}{cc}
\mathbf{B}_{l} & \mathbf{0}_{l \times (n-l)} \\
\mathbf{0}_{(n-l) \times l} & \mathbf{B}_{n-l}
\end{array} \right]
\end{equation}
\end{small}\noindent
where $\mathbf{0}_{i \times j}$ is a $i \times j$ matrix of all zeros, and $\mathbf{A}_i$ and $\mathbf{B}_i$ are respective $i \times i$ adjacency and combinatorial graph Laplacian matrices for full unweighted graph of $i$ vertices:

\vspace{-0.1in}
\begin{footnotesize}
\begin{equation}
\mathbf{A}_i = \left[ \begin{array}{cccc}
0 & 1 & \ldots & \\
1 & 0 & 1 & \ldots \\
\vdots & & \ddots &
\end{array} \right],
~
\mathbf{B}_i = \left[ \begin{array}{cccc}
i-1 & -1 & \ldots & \\
-1 & i-1 & -1 & \ldots \\
\vdots & & \ddots &
\end{array} \right]
\end{equation}
\end{footnotesize}

Clearly, normalized $\LL_n = \mathbf{D}^{-1/2}\L\mathbf{D}^{-1/2}$ is also block-diagonal:
\begin{equation}
\LL_n = \left[ \begin{array}{cc}
\tilde{\mathbf{B}}_{l} & \mathbf{0}_{l \times (n-l)} \\
\mathbf{0}_{(n-l) \times l} & \tilde{\mathbf{B}}_{n-l}
\end{array} \right]
\label{eq:normL_PWC}
\end{equation}
where $\tilde{\mathbf{B}}_i$ is a $i \times i$ normalized graph Laplacian for a full unweighted graph of $i$ vertices:

\vspace{-0.1in}
\begin{small}
\begin{equation}
\tilde{\mathbf{B}}_i = \left[ \begin{array}{cccc}
1 & -1/(i-1) & \ldots & \\
-1/(i-1) & 1 & -1/(i-1) & \ldots \\
\vdots & & \ddots &
\end{array} \right]
\label{eq:normL_full}
\end{equation}
\end{small}

We already know $\v_1 = \D^{1/2} \mathbf{1}$ is the first eigenvector corresponding to eigenvalue $0$.
Further, we see also that a second (unnormalized) eigenvector $\mathbf{v}_2$ corresponding to eigenvalue $0$ can be constructed as:
\begin{equation}
v_{2,i} = \left\{ \begin{array}{ll}
1/l(l-1)^{1/2} & \mbox{if}~ 1 \leq i \leq l \\
-1/(n-l)(n-l-1)^{1/2} & \mbox{if}~ l < i \leq n
\end{array} \right.
\end{equation}
It is easy to verify that $\mathcal{L}_n \mathbf{v}_2 = \mathbf{0}$ and $\mathbf{v}_1^T \mathbf{v}_2 = 0$.
In fact, we see that $\f^* = \D^{-1/2} \v_2$ takes the form in (\ref{eq:relaxed_ncut2}); \textit{i.e.}, for this graph corresponding to an ideal PWC signal, the relaxed solution $\v_2$ to Ncut is also the optimal one, which is PWC.


Analyzing coefficients $\boldsymbol{\beta} = \mathbf{V}^T \D^{1/2} \x$, we see that $\D^{1/2} \mathbf{x}$ can be expressed \textit{exactly} as a linear combination of the first two eigenvectors of $\mathcal{L}_n$, \textit{i.e.}, $\D^{1/2} \mathbf{x} = a_1 {\v_1} + a_2 {\v_2}$, where
 \begin{equation}
 a_1 = \frac{c_1l(l-1)+c_2(n-l)(n-l-1)}{{(n - l)(n - l - 1) + l(l - 1)}},
  \end{equation}
and
 \begin{equation}
 \small
  a_2 =  \frac{(c_1 - c_2)l(l-1)(n-l)(n-l-1)}{{(n - l)(n - l - 1) + l(l - 1)}}.
  \end{equation}
Because the first two eigenvectors $\v_1$ and $\v_2$ correspond to eigenvalue $0$ of $\mathcal{L}_n$, we can conclude:

\begin{quote}
\textit{Given an ideal two-piece PWC signal $\x$ defined in (\ref{eq:pwc_sig}), $\D^{1/2} \x$ can be represented exactly using the first two eigenvectors of $\mathcal{L}_n$ corresponding to eigenvalue $0$, and hence LERaG evaluates to $0$.}
\end{quote}

In other words, $\D^{1/2} \x$ is also ideal low-pass (cut-off frequency at 0) given eigenvectors of $\mathcal{L}_n$, and so LERaG, which penalizes high frequencies according to (\ref{eq:freqInterReg}), computes to $0$.
Another interpretation is that an ideal two-piece PWC signal leads to two clusters of identical vertices, hence $\v_2$ of $\mathcal{L}_n$ is an exact solution to Ncut, which is PWC (\ref{eq:relaxed_ncut2}).
In this case, $\v_1$ and $\v_2$ are both PWC, so together they are sufficient to represent $\D^{1/2} \x$, also PWC.

\subsubsection{Analysis of Ideal Piecewise Smooth Signals}

We now generalize our analysis of PWC signals to PWS signals. We first define a PWS signal $\mathbf{x}$ as:
\begin{equation}
\begin{array}{ll}
|x_i - x_j| \leq \delta & \mbox{if}~~ i,j \leq l ~~\mbox{or}~~ i,j > l \\
|x_i - x_j| > \Delta & \mbox{if}~~ i \leq l < j ~~\mbox{or}~~ j \leq l < i
\end{array}
\label{eq:pws_sig}
\end{equation}
where $\delta \ll \Delta$.
In words, $\mathbf{x}$ contains two smooth pieces in $[1, l]$ and $[l+1, N]$: two samples in the same piece are similar to within $\delta$, and two samples from two different pieces differ by more than $\Delta$.
Clearly, if $\delta = 0$, then $\mathbf{x}$ is PWC.

We now construct a complete graph for $\mathbf{x}$ using (\ref{eq:weight}) to compute edge weight $W_{i,j}$ between vertices $i$ and $j$.
The resulting normalized $\mathcal{L}_n$ is still block-diagonal for sufficiently large $\Delta$, but the diagonal blocks are no longer $\tilde{\mathbf{B}}_i$ in (\ref{eq:normL_full}).
$\v_2$ in this case can be computed as:
\begin{equation}
v_{2,i} = \left\{ \begin{array}{ll}
\frac{D_{i,i}^{1/2}}{\sum_{j=1}^l D_{j,j}} & \mbox{if}~ 1 \leq i \leq l \\
- \frac{D_{i,i}^{1/2}}{\sum_{j=l+1}^n D_{j,j}} & \mbox{if}~ 1 < i \leq n
\end{array} \right.
\label{eq:evec2PWS}
\end{equation}

One can verify that $\mathcal{L}_n \v_2 = \mathbf{0}$ and $\v_2^T \v_1 = 0$.
$\v_2$ is roughly PWS according to (\ref{eq:evec2PWS}), since $D_{i,i}^{1/2}$ is similar for nodes $i$ within the same smooth piece.
Thus $D^{1/2} \x$, also roughly PWS, can be well approximated as $a_1 \v_1 + a_2 \v_2$, with little energy leakage to high frequencies; \textit{i.e.}, $\D^{1/2} \x$ is approximately low-pass given eigenvectors of $\mathcal{L}_n$.
That means LERaG computes to a small value for PWS signal $\x$.

In the more general case when signal $\x$ is not PWS, then $\mathcal{L}_n$ is not block-diagonal, and second eigenvector $\v_2$ computed in (\ref{eq:relaxed_cut4}) corresponds to an eigenvalue $\tilde{\eta}_2 > 0$.
This first non-zero eigenvalue $\tilde{\eta}_2$ is called the \textit{Fiedler number} \cite{SPIELMAN}, and is a measure of how connected the graph is.
In our signal restoration context, it means that the closer $\x$ is to PWS, the closer $\tilde{\eta}_2$ is to $0$, and the more likely LERaG can recover the signal $\x$ well.


\begin{figure}[t!]
\centering
\includegraphics[width = 0.5\textwidth]{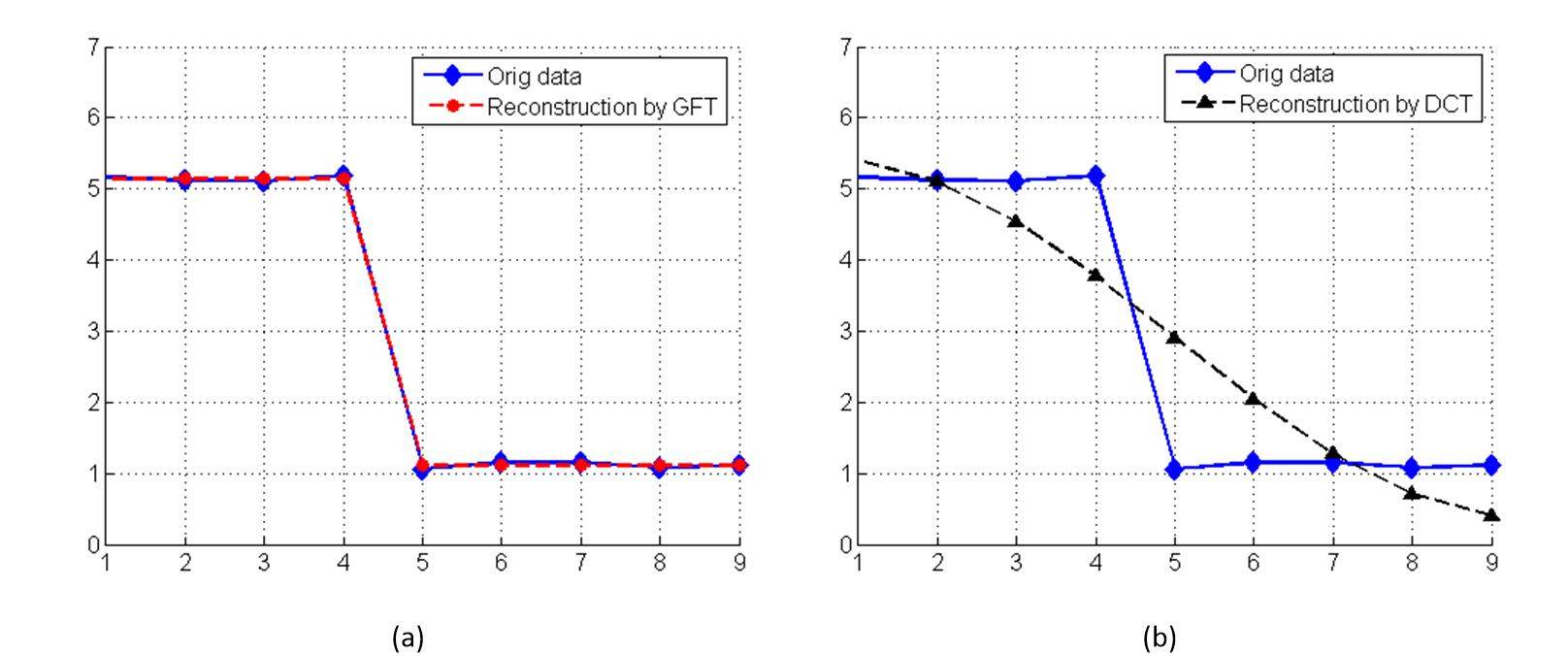}
\caption{Signal reconstruction comparison using only the first two right eigenvectors of $\mathcal{L}_r$ and DCT. }
\label{fig:rec}
\vspace{-0.3cm}
\end{figure}


For intuition, we illustrate a simple 1-D PWS signal, where $\delta = 0.2$ and $\Delta = 4$. 
Fig.~\ref{fig:rec} (a) shows an example where the first two right eigenvectors of $\mathcal{L}_r$ can represent the original signal well.
In contrast, as illustrated in Fig.~\ref{fig:rec} (b), reconstruction using only the first two DCT frequencies is poor.

\section{Soft Decoding based on Prior Mixture}
\label{sec:optimization}
We now combine the priors discussed in previous sections to construct an efficient JPEG soft decoding algorithm.

\subsection{The Objective Function Formulation}

Combining the sparsity prior and LERaG, we arrive at the following optimization for soft decoding:
\begin{equation}
\label{eq:obj_priors}
\begin{array}{l}
\mathop{\arg \min} \limits_{\{\x, \a \}}  \left\| {\x - \bm{\Phi} \a } \right\|_2^2 + {\lambda_1}{\left\| \a  \right\|_0} + \lambda_2\x^T (d_{\min}^{-1})\mathbf{L} \D^{-1} \mathbf{L}\x, \\
\mbox{s.t.}~ \q\Q \preceq \mathbf{T}\M\x \prec (\q+1)\Q
\end{array}
\end{equation}
where both patch $\x$ and its sparse code $\a$ are unknown.

For parameter $\lambda_1$ of the sparsity prior, we empirically set $\lambda_1 = 0.001$. For parameter $\lambda_2$ of the graph-signal smoothness prior, we set it adaptively as follows.
As argued in Section IV-C, if a patch contains large high DCT frequencies, then it will not be recovered well by the sparsity prior.
We thus adaptively increase $\lambda_2$ if q-bin indices $\mathbf{q}$ indicate the presence of high DCT frequencies in target $\x$.


Since by design adjacent patches have overlaps, the pixel values in an overlapping region from two patches should be as similar as possible. To this end, upon obtaining all the optimal reconstructed patches $\{\x_i^*\}$, the soft-decoded image can then be obtained by averaging over pixels in overlapped regions.

\subsection{Optimization}

Optimization of (\ref{eq:obj_priors}) is not convex. Instead, we propose to employ an alternating procedure to optimize $\x$ and $\a$ iteratively. First, for initialization we compute the Laplacian prior based MMSE solution per code block for the entire image to obtain $\x^{(0)}$.
Then, the $l$-th ($l>0$) iteration of the optimization procedure is described as follows:

1) Fix $\x$ and estimate $\a$:

The optimization problem becomes a standard sparse coding:
\begin{equation}
\label{eq:standard}
\a^{(l)} = \mathop{\arg \min} \limits_{\a} \left\| {\x^{(l-1)} - \bm{\Phi} \a } \right\|_2^2 + {\lambda_1}{\left\| \a  \right\|_0},
\end{equation}
which can be efficiently solved by some well-developed greedy $\ell_0$-minimization algorithm such as OMP.

2) Fix $\a$ and estimate $\x$:

The objective function becomes:
\begin{equation}
\label{eq:obj2}
\begin{array}{l}
\x^{(l)} = \mathop{\arg \min} \limits_{\x} \left\| {\x - \bm{\Phi} \a^{(l)}} \right\|_2^2 + \lambda_2\x^T (d_{\min}^{-1})\mathbf{L} \D^{-1} \mathbf{L}\x,\\
 \mbox{s.t.}~ \q\Q \preceq \mathbf{T}\M\x \prec (\q+1)\Q
\end{array}
\end{equation}
which can be solved efficiently using quadratic programming \cite{Nocedal06}.

The outputted solution $\x^{(l)}$ is used as the initial for $(l+1)$-th iteration. The graph $\G$ is updated according to $\x^{(l)}$. The algorithm terminates when both $\x$ and $\a$ converge. The local convergence can be proved in a similar way as \textbf{Lemma 1}.

\begin{table*}[t!]
\begin{center}
\renewcommand{\arraystretch}{1.3}
\tabcolsep=5.7pt
\caption{Quality comparison with respect to PSNR (in dB) and SSIM at QF = 5}
\label{tab:PSNR1}
\begin{tabular}{|c||c|c||c|c||c|c||c|c||c|c||c|c||c|c||}
\hline
\multirow{2}{*}{Images} &\multicolumn{2}{c||}{JPEG} &\multicolumn{2}{c||}{BM3D \cite{BM3DSAPCA}} &\multicolumn{2}{c||}{KSVD \cite{KSVD}} &\multicolumn{2}{c||}{ANCE \cite{ZhangTIP}}&\multicolumn{2}{c||}{DicTV \cite{DicTV}}&\multicolumn{2}{c||}{SSRQC \cite{SSRQC}} &\multicolumn{2}{c||}{Ours}
\\ \cline{2-3} \cline{4-5}\cline{6-7}\cline{8-9}\cline{10-11}\cline{12-13}\cline{14-15}
           &PSNR &SSIM &PSNR &SSIM &PSNR &SSIM &PSNR &SSIM &PSNR &SSIM &PSNR &SSIM &PSNR &SSIM\\\hline
\emph{Butterfly}	&22.65	&0.7572	&23.91	&0.8266	&24.55	&0.8549 &24.34	&0.8532	&23.54	&0.8228	&25.31	&0.8764	&\textbf{25.82}	&\textbf{0.8861}\\\hline
\emph{Leaves}	    &22.49	&0.7775	&23.78	&0.8408	&24.39	&0.8684 &24.18	&0.8551	&23.27	&0.8245	&25.01	&0.8861	&\textbf{25.57}	&\textbf{0.8979}\\\hline
\emph{Hat}	        &25.97	&0.7117	&26.79	&0.7497	&27.41	&0.7802 &27.09	&0.7706	&27.33	&0.7707	&27.27	&0.7753	&\textbf{27.71}	&\textbf{0.7946}\\\hline
\emph{Boat}	        &25.23	&0.7054	&26.31	&0.7547	&26.85	&0.7739 &26.59	&0.7637	&26.31	&0.7491	&26.85	&0.7745	&\textbf{27.17}	&\textbf{0.7783}\\\hline
\emph{Bike}	        &21.72	&0.6531	&22.60	&0.7039	&22.84	&0.7130 &22.74	&0.6973	&22.28	&0.6952	&22.96	&0.7119	&\textbf{23.32}	&\textbf{0.7291}\\\hline
\emph{House}	    &27.76	&0.7732	&28.87	&0.8020	&29.53	&0.8185 &29.07	&0.8131	&29.59	&0.8072	&29.92	&0.8226	&\textbf{30.25}	&\textbf{0.8237}\\\hline
\emph{Flower}	    &24.51	&0.6866	&25.49	&0.7352	&25.84	&0.7471 &25.54	&0.7337	&25.88	&0.7316	&25.69	&0.7407	&\textbf{25.93}	&\textbf{0.7521}\\\hline
\emph{Parrot}	    &26.15	&0.7851	&27.40	&0.8329	&27.22	&0.8465	&27.81	&0.8475	&27.92	&0.8382	&28.21	&0.8566	&\textbf{28.25	}&\textbf{0.8572}\\\hline
\emph{Pepper512}	&27.17	&0.7078	&28.31	&0.7573	&29.32	&0.7949 &29.03	&0.7891	&28.81	&0.7769	&29.28	&0.7948	&\textbf{29.81}	&\textbf{0.7984}\\\hline
\emph{Fishboat512}	&25.56	&0.6563	&26.44	&0.6921	&26.87	&0.7072 &26.72	&0.6994	&26.35	&0.6963	&26.86	&0.7042	&\textbf{27.04}	&\textbf{0.7113}\\\hline
\emph{Lena512}	    &27.32	&0.7365	&28.43	&0.7788	&29.13	&0.8079 &28.96	&0.8024	&28.51	&0.7976	&29.19	&0.8072	&\textbf{29.29}	&\textbf{0.8077}\\\hline
\emph{Airplane512}	&26.34	&0.7843	&27.11	&0.8101	&27.62	&0.8262 &27.40	&0.8201	&26.95	&0.8112	&28.90	&0.8261	&\textbf{29.19}	&\textbf{0.8482}\\\hline
\emph{Bike512}	    &22.25	&0.6231	&23.12	&0.6693	&23.39	&0.6781 &23.24	&0.6605	&22.73	&0.6559	&24.60	&0.6815	&\textbf{24.86}	&\textbf{0.6992}\\\hline
\emph{Statue512}	&25.72	&0.6629	&33.61	&0.9188	&27.03	&0.7371 &26.83	&0.7265	&26.51	&0.7268	&28.03	&0.7627	&\textbf{28.30}	&\textbf{0.7685}\\\hline\hline
Average	            &25.06	&0.7157	&26.58	&0.7765	&26.62	&0.7827	&26.39	&0.7737	&26.14	&0.7645	&27.01	&0.7872	&\textbf{27.32}	&\textbf{0.7941}\\\hline
\end{tabular}
\end{center}
\vspace{-0.3cm}
\end{table*}

\begin{table*}[t!]
\begin{center}
\renewcommand{\arraystretch}{1.3}
\tabcolsep=5.7pt
\caption{Quality comparison with respect to PSNR (in dB) and SSIM at QF = 40}
\label{tab:PSNR2}
\begin{tabular}{|c||c|c||c|c||c|c||c|c||c|c||c|c||c|c||}
\hline
\multirow{2}{*}{Images} &\multicolumn{2}{c||}{JPEG} &\multicolumn{2}{c||}{BM3D \cite{BM3DSAPCA}} &\multicolumn{2}{c||}{KSVD \cite{KSVD}} &\multicolumn{2}{c||}{ANCE \cite{ZhangTIP}}&\multicolumn{2}{c||}{DicTV \cite{DicTV}}&\multicolumn{2}{c||}{SSRQC \cite{SSRQC}} &\multicolumn{2}{c||}{Ours}
\\ \cline{2-3} \cline{4-5}\cline{6-7}\cline{8-9}\cline{10-11}\cline{12-13}\cline{14-15}
           &PSNR &SSIM &PSNR &SSIM &PSNR &SSIM &PSNR &SSIM &PSNR &SSIM &PSNR &SSIM &PSNR &SSIM\\\hline
\emph{Butterfly}	&29.97	&0.9244	     &31.35	&0.9555	  &31.57	&0.9519        &31.38	&0.9548	&31.22	&0.9503	&32.02	&0.9619	   &\textbf{32.87}	&\textbf{0.9627}\\\hline
\emph{Leaves}	    &30.67	&0.9438	     &32.55	&0.9749	  &33.04	&0.9735        &32.74	&0.9728	&32.45	&0.9710	&32.13	&0.9741	   &\textbf{34.42}	&\textbf{0.9803}  \\\hline
\emph{Hat}	        &32.78	&0.9022	     &33.89	&0.9221	  &33.62	&0.9149        &33.69	&0.9169	&33.20	&0.8988	&34.10	&0.9237	   &\textbf{34.46}	&\textbf{0.9268}     \\\hline
\emph{Boat}	        &33.42	&0.9195	     &34.77	&0.9406   &34.28	&0.9301	       &34.64	&0.9362	&26.08	&0.7550	&33.88	&0.9306	   &\textbf{34.98}	&\textbf{0.9402}       \\\hline
\emph{Bike}	        &28.98	&0.9131	     &29.96	&0.9356   &30.19	&0.9323	       &30.31	&0.9357	&29.75	&0.9154	&30.35	&0.9411	   &\textbf{31.14}	&\textbf{0.9439}        \\\hline
\emph{House}	    &35.07	&0.8981	     &36.09	&0.9013	  &36.05	&0.9055        &36.12	&0.9048	&35.17	&0.8922	&36.49	&0.9072	   &\textbf{36.55}	&\textbf{0.9071}  \\\hline
\emph{Flower}	    &31.62	&0.9112	     &32.81	&0.9357	  &32.63	&0.9271        &32.67	&0.9314	&31.86	&0.9084	&33.02	&0.9362	   &\textbf{33.37}	&\textbf{0.9371}    \\\hline
\emph{Parrot}	    &34.03	&0.9291	     &34.92	&0.9397	  &34.91	&0.9371        &35.02	&0.9397	&33.92	&0.9227	&35.11	&0.9401	   &\textbf{35.32}	&\textbf{0.9401}     \\\hline
\emph{Pepper512}	&34.21	&0.8711	     &34.94	&0.8767	  &34.89	&0.8784        &34.99	&0.8803	&34.24	&0.8639	&35.05	&0.8795	   &\textbf{35.19}	&\textbf{0.8811}  \\\hline
\emph{Fishboat512}  &32.76	&0.8763	     &33.61	&0.8868   &33.36	&0.8809	       &33.60	&0.8861	&32.53	&0.8496	&33.68	&0.8859	   &\textbf{33.73}	&\textbf{0.8871}        \\\hline
\emph{Lena512}	    &35.12	&0.9089	     &36.03	&0.9171	  &35.82	&0.9146        &36.04	&0.9177	&34.85	&0.8986	&36.09	&0.9187	   &\textbf{36.11}	&\textbf{0.9191}    \\\hline
\emph{Airplane512}  &33.36	&0.9253	     &34.38	&0.9361   &34.36	&0.9341	       &34.53	&0.9358	&33.75	&0.9134	&35.81	&0.9355	   &\textbf{36.07}	&\textbf{0.9439 }     \\\hline
\emph{Bike512}	    &29.43	&0.9069	     &30.47	&0.9299	  &30.66	&0.9258        &30.71   &0.9298	&30.05  &0.9043	&32.26  &0.9372    &\textbf{32.55}&\textbf{0.9387}        \\\hline
\emph{Statue512}	&32.78	&0.9067	     &33.61	&0.9188	  &33.55	&0.9149        	&33.55	&0.9193	&32.53	&0.8806	   &34.88	&0.9249&\textbf{34.95}	&\textbf{0.9273}   \\\hline\hline
Average	            &32.44	&0.9097	     &33.52	&0.9264   &33.50	&0.9229	       	&33.57	&0.9258	&32.25	&0.8945    &33.91	&0.9283	   &\textbf{34.41}	&\textbf{0.9311}        \\\hline
\end{tabular}
\end{center}
\vspace{-0.3cm}
\end{table*}

\section{Experimentation}
\label{sec:results}

\begin{figure}[t!]
\centering
\includegraphics[width = 0.5\textwidth]{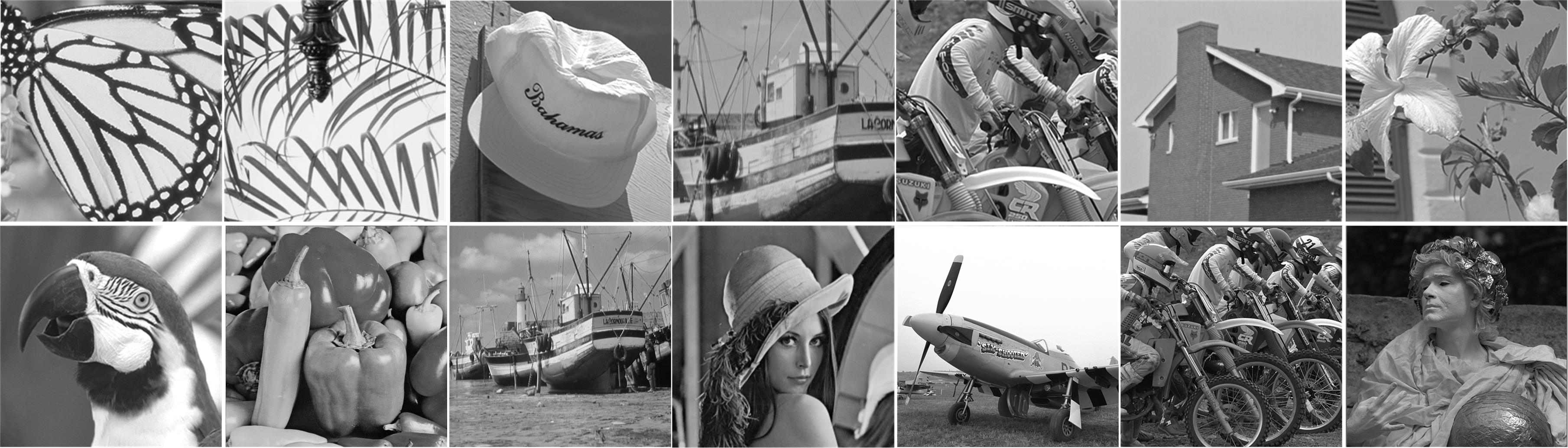}
\caption{Fourteen test images. The image order is the same as Table I and II.}
\label{fig:3-testimg}
\vspace{-0.3cm}
\end{figure}

We now present experimental results to demonstrate the superior performance of our soft decoding algorithm.
We selected fourteen widely used images in the literature as test images. The first eight images are sized of $256 \times 256$, the rest six images are sized of $512 \times 512$, as illustrated in Fig.~\ref{fig:3-testimg}.

\subsection{Competing Methods}


Our proposed method is compared against: 1) BM3D \cite{BM3D}: a state-of-the-art image denoising method, which is included because the restoration of compressed images can be viewed as a denoising problem; in our test, the extended version \cite{BM3DSAPCA} is used, which achieves better performance than original BM3D.
2) KSVD \cite{KSVD}: a well-known sparse coding framework; we test KSVD with a large enough over-complete dictionary ($100 \times 4000$). In contrast, our scheme uses a much smaller one ($100 \times 400$). We will demonstrate that our scheme achieves better reconstruction performance along with lower computational complexity compared with KSVD.
3) ANCE \cite{ZhangTIP}: a state-of-the-art JPEG image restoration algorithm.
4) DicTV \cite{DicTV}: a recent sparsity-based compressed image restoration algorithm, which exploits both sparsity and TV priors.
5) SSRQC \cite{SSRQC}: a state-of-the-art JPEG soft decoding method published very recently, which can be regarded as one of the best algorithms to-date.
The source codes of compared methods are all kindly provided by their authors.

The determination of parameters deserves clarification.
Regarding the size of patches, larger patches can result in more overlapping pixels among neighboring $8 \times 8$ code blocks, which will be beneficial in removing block artifcacts.
However, it will also make the task of finding their sparse representations difficult.
Thus we set the size of the patch to be $10 \times 10$, overlapping 2 pixels horizontally and vertically with neighboring patches, to achieve a good tradeoff.
For the uncompressed training set used to obtain the over-complete dictionary, we randomly select five images from the Kodak Lossless True Color Image Suite\footnote{ http://r0k.us/graphics/kodak/}. The training set does not have any overlap with the test set.

Our method can be easily extended to restore compressed color images.
When compressing color images, JPEG first performs YUV color transformation, and then compresses the resulting Y, U and V channels separately.
As the image signal energy is highly packed into the luminance channel Y, we apply our proposed scheme to Y.
For chrominance channels U and V, we only use the graph-signal smoothness prior to speed up the restoration process.
Due to space limitation, we only report the test results on gray-level images in the following.


\begin{figure*}[t!]
\centering
\includegraphics[width = 0.99\textwidth]{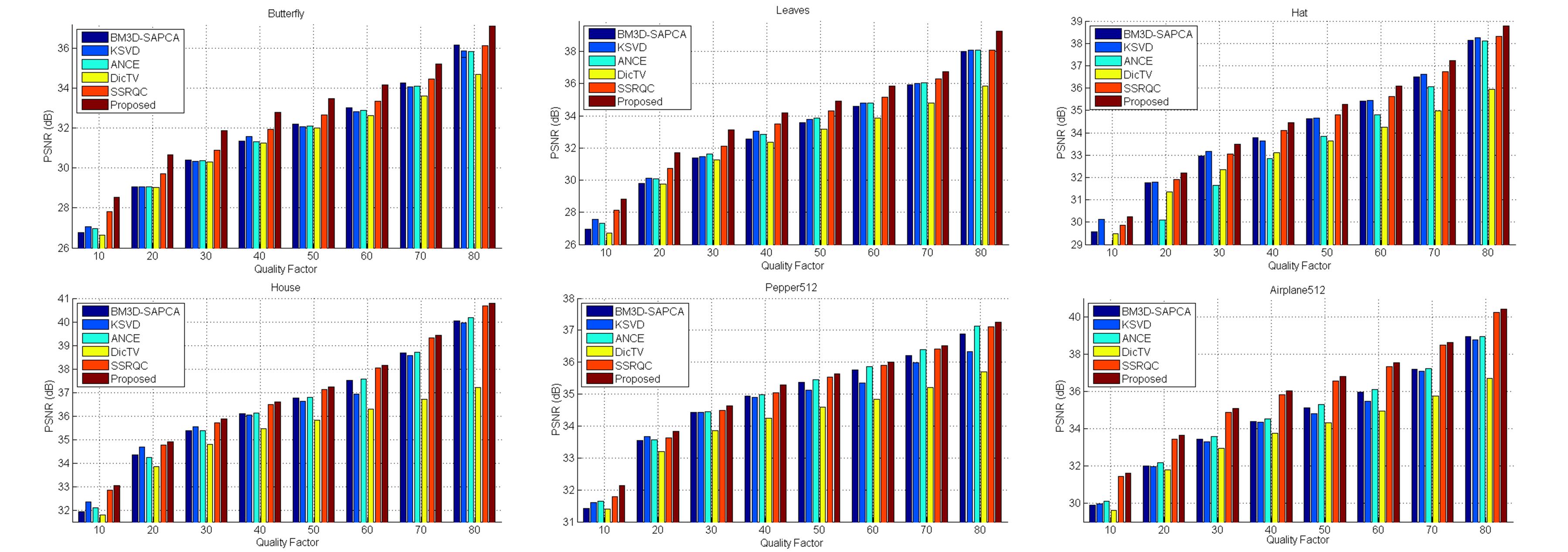}
\caption{ QF-PSNR performance comparison for QFs ranging from 10 to 80. The proposed method achieves the best PSNR performance on all QFs.}
\label{fig:3-psnrcurve}
\vspace{-0.3cm}
\end{figure*}

\subsection{Objective Performance Comparison}

Table~\ref{tab:PSNR1} and Table~\ref{tab:PSNR2} tabulate the objective evaluations (PSNR and SSIM) of the above algorithms for fourteen test images.
The test images are coded by a JPEG coder with with a small QF = 5 and a medium QF = 40.

When QF = 5, the average PSNR gains are 0.74dB and 0.7dB respectively, compared to two well-known algorithms BM3D and KSVD. Note that KSVD uses a large enough dictionary for reconstruction. While our method uses a much smaller dictionary, it achieves better results than KSVD, which demonstrate the effectiveness of the proposed graph-signal smoothness prior. Compared to the state-of-the-art ANCE algorithm, our method greatly improves the reconstruction quality. The PSNR gain is up to 1.48dB (\emph{Butterfly}) and the average PSNR gain is 0.93dB. Our method also performs better than state-of-the-art sparsity based methods.
Compared to DicTV, designed specifically to restore compressed images, the average gain of our method is 1.18dB.
Compared to SSRQC, the newest sparsity-based soft decoding method, our method achieves an average gain of 0.31dB.
When QF = 40, the proposed algorithm also performs better than other methods for all test images.
The average PSNR gains are 0.89dB, 0.91dB, 0.84dB, 2.16dB and 0.5dB, respectively.

We further show the test results for a wide range of QFs (from 10 to 80), illustrated in Fig.~\ref{fig:3-psnrcurve}.
Due to the space limitation, here we only report the test results over six test images, including four $256 \times 256$ images and two $512 \times 512$ images.
We observe that our method consistently performs better than other methods.


\subsection{Subjective Performance Comparison}


We now demonstrate that our soft decoding algorithm also achieves better perceptual quality of the restored images.
When QF is 5, the quantization noise is severe, which leads to very poor subjective quality of JPEG-compressed images.
Therefore, we use QF = 5 as an example to evaluate visual quality of different schemes.

Fig.~\ref{fig:3-butterfly} and Fig.~\ref{fig:3-leaves} illustrate the perceptual quality for different methods.
The test images \emph{Butterfly} and \emph{Leaves} have rich structure regions.
We use them as examples to demonstrate the performance of different methods in recovering structure regions.
The images reproduced by BM3D suffer from highly visible noises, especially the blocking artifacts.
KSVD achieves better results than BM3D, where fewer blocking artifacts are detectable.
However, KSVD cannot preserve the edge structure well.
In results produced by DicTV, there are still strong blocking artifact.
This is because, in DicTV, the dictionary is learned from the JPEG compressed image itself.
When quantization step is large, the structure noise is also learned as atoms of dictionary.
Therefore, it will enhance but not suppress the quantization noise through subsequent sparse coding based restoration.
ANCE can suppress most blocking artifacts, but there are still noticeable artifacts along edges.
SSRQC removes most block artifacts.
However, it can be observed that the edge regions are blurred to some extent.

The images restored by our method are much cleaner, in which the structures and sharpness of edges are well preserved.
Our proposed method primarily benefit from exploiting both sparsity prior for recovering textures (low DCT frequencies) and graph-signal smoothness prior for recovering structures (high DCT frequencies).
The experimental results validate this point.
Our proposed method can also remove DCT blocking artifacts in smooth areas completely, which are largely free of the staircase and ringing artifacts along edges.
Due to space limitation, here we only show the subjective comparisons for low QF, as the superiority can be better visually reflected in these cases. Our method also achieves better subjective quality for medium to high QFs as well.

\begin{figure*}[ht!]
\centering
\includegraphics[width = 0.75\textwidth]{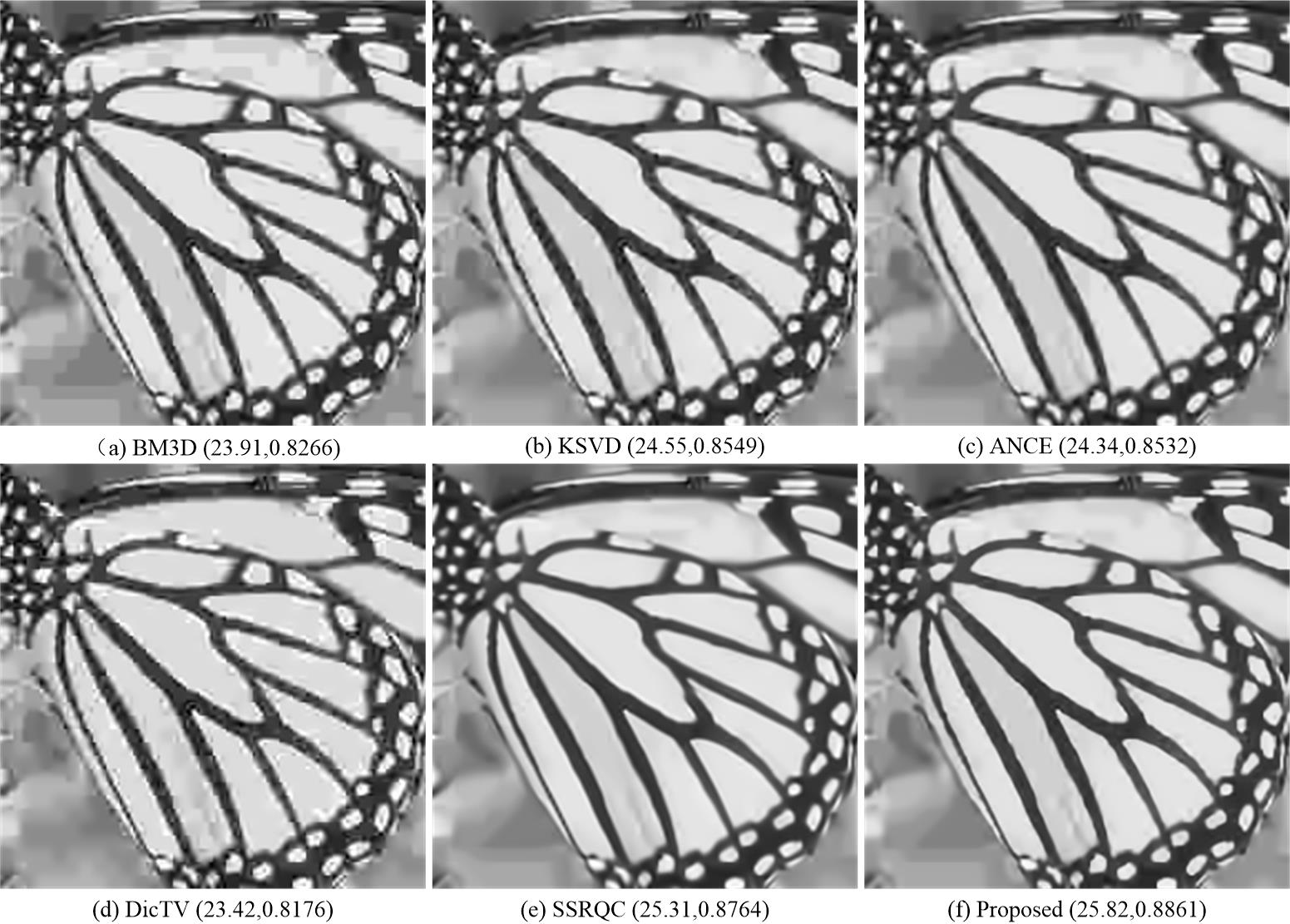}
\caption{Comparison of tested methods in visual quality on \emph{Butterfly} at QF = 5. The corresponding PSNR and SSIM values are also given as references.}
\label{fig:3-butterfly}
\vspace{-0.1cm}
\end{figure*}

\begin{figure*}[ht!]
\centering
\includegraphics[width = 0.75\textwidth]{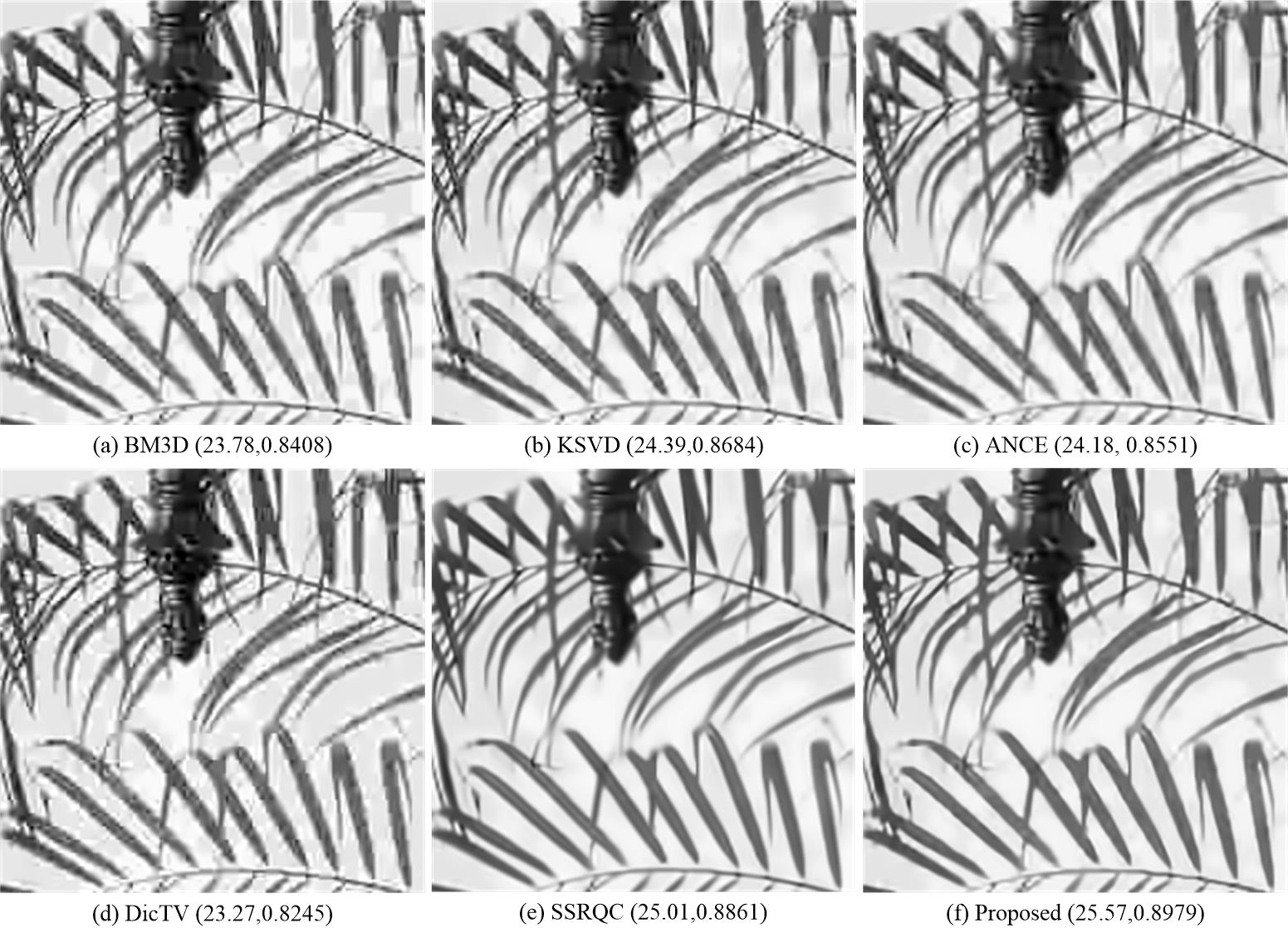}
\caption{Comparison of tested methods in visual quality on \emph{Leaves} at QF = 5. The corresponding PSNR and SSIM values are also given as references.}
\label{fig:3-leaves}
\vspace{-0.3cm}
\end{figure*}

\subsection{Computational Complexity Comparison}

We further report the results on computational complexity comparison.
Here we show the average running time comparison on eight $256 \times 256$ test images when QF = 40.
The compared methods are running on a typical laptop computer (Intel Core i7 CPU \@2.6GHz, 8G Memory, Win10, Matlab R2014a).
As depicted in Table~\ref{tab:TIME}, the complexity of our method is lower than the state-of-the-art algorithms BM3D (BM3D-SAPCA) and ANCE. Compared with K-SVD, which is with a large enough over-complete dictionary, the proposed scheme can significantly reduce the computational complexity.

\begin{table}[htbp]
\renewcommand{\arraystretch}{1.6}
\tabcolsep=3.8pt
  \centering
  \caption{Computation Complexity Comparison (in second) when QF = 40}
    \begin{tabular}{c||c|c|c|c|c|c}\hline
 TIME & BM3D & KSVD  &ANCE & DicTV &SSRQC &Proposed \\\hline
Average & 373.35 & 209.71 &307.43 &39.53 &70.32 & 143.73 \\\hline
    \end{tabular}%
  \label{tab:TIME}%
  \vspace{-0.4cm}
\end{table}%

\subsection{ Comparison of Different Graph-signal Smoothness Prior}

To demonstrate the superiority of the proposed graph Laplacian regularizer LERaG, we compare it with the popular combinational graph Laplacian regularizer, symmetrically normalized graph Laplacian regularizer and the doubly stochastic one proposed in \cite{PaymanTIP14}.
We test eight $256 \times 256$ images when QF = 5.
The results reported are PSNR values after the first iteration.
That is, after solving (\ref{eq:standard}), we use three different graph Laplacian regularizer to formulate (\ref{eq:obj2}).
For fairness of comparison, the regularization parameter $\lambda_2$ are carefully selected for optimal performance for each case.
It can be found that the proposed regularizer outperform other three ones with respect to average PSNR.
Compared with other three graph Laplacian regularizers, our method can improve the PSNR performance up to 0.18dB, 0.87dB and 0.42dB, respectively.


\begin{table}[ht]
\begin{center}
\renewcommand{\arraystretch}{1.1}
\tabcolsep=5pt
\caption{Performance comparison of different graph-signal smooth priors with respect to PSNR (in dB) and SSIM at QF = 5}
\label{tab:PSNR3}
\begin{tabular}{|c||c|c|c|c||}
\hline
    Images       &Combinatorial &Normalized &Doubly Stochastic &LERaG \\\hline
\emph{Butterfly}	&25.42	&24.70	   &25.15	  	   &\textbf{25.57}	                  \\\hline
\emph{Leaves}	    &24.99	&24.54	&24.84	 	   &\textbf{25.17}	                    \\\hline
\emph{Hat}	        &27.53	&27.42	&27.43		   &\textbf{27.56}	                 \\\hline
\emph{Boat}	        &26.99	&26.94	&26.98	       &\textbf{26.99}	                \\\hline
\emph{Bike}	        &23.12	&23.01	&23.09         &\textbf{23.17}	                \\\hline
\emph{House}	    &29.87	&29.83	&29.86	  	   &\textbf{29.89}	                     \\\hline
\emph{Flower}	    &25.84	&25.78	&25.82	 	   &\textbf{25.87}	                     \\\hline
\emph{Parrot}	    &27.97	&27.95	&27.97	 	   &\textbf{28.02}	                   \\\hline\hline
Average 	        &26.46	&26.27	&26.39	  	   &\textbf{26.53}	                     \\\hline
\end{tabular}
\end{center}
\vspace{-0.3cm}
\end{table}


\section{Conclusion}
\label{sec:conclude}

In this paper, a novel soft decoding approach for the restoration of JPEG-compressed images is proposed. The main technical contribution is twofold. First, we propose a new graph-signal smoothness prior based on left eigenvectors of the random walk graph Laplacian with desirable image filtering properties, which can recover high DCT frequencies of piecewise smooth signals well. 
Second, we combine the Laplacian prior, sparsity prior and our new graph-signal smoothness prior into an efficient soft-decoding algorithm. 
Experimental results demonstrate that our method achieves better objective and subjective restoration quality compared to state-of-the-art soft decoding algorithms.

\bibliographystyle{IEEEtran}
\bibliography{ref}

\end{document}